\documentclass[1p]{elsarticle}

\usepackage[colorlinks = true]{hyperref}
\usepackage{microtype}
\usepackage{graphicx}
\usepackage{subfigure}
\usepackage{booktabs} 
\usepackage{amsmath}
\usepackage{amssymb}
\usepackage{amsthm}
\usepackage{colortbl}
\usepackage{dsfont}
\usepackage{footmisc}
\usepackage[dvipsnames]{xcolor}
\usepackage{tablefootnote}
\usepackage[export]{adjustbox}

\usepackage{tabularx} 

\newcommand{\weight}{\lambda}
\newcommand{\Ac}{A_c}
\newcommand{\agcn}[1]{\mathcal{F}_{GCN}(#1)} 
\newcommand{\ncluster}{K}
\newcommand{\clustersubscript}{k}
\newcommand{\dregc}{s}
\newcommand{\Ao}{A_\dregc}
\newcommand{\Acomp}{A'}
\newcommand{\ncomp}{n'}
\newcommand{\Accomp}{\Ac'}

\newcommand{\lossGAE}{\mathcal{L}_{\text{GAE}}}
\newcommand{\lossVGAE}{\mathcal{L}_{\text{VGAE}}}
\newcommand{\lossMAGAE}{\tilde{\mathcal{L}}_{\text{GAE}}}
\newcommand{\lossMAVGAE}{\tilde{\mathcal{L}}_{\text{VGAE}}}

\newcommand{\mati}{\Phi}
\newcommand{\matii}{\Psi}
\newcommand{\evali}{\phi}
\newcommand{\evalii}{\psi}
\newcommand{\eveci}{u}

\newcommand{\eqnref}[1]{(\ref{#1})}

\newtheorem{theorem}{Theorem}

\theoremstyle{definition}
\newtheorem{proposition}[theorem]{Proposition}

\newtheorem{definition}[theorem]{Definition}

\journal{Neural Networks (accepted: June 2022)}

\begin{document}

\begin{frontmatter}

\title{Modularity-Aware Graph Autoencoders\\ for Joint Community Detection and Link Prediction}

\author[1,2]{Guillaume Salha-Galvan\corref{cor1}}
\author[2]{Johannes F. Lutzeyer}
\author[2]{George~Dasoulas}
\author[1]{Romain~Hennequin}
\author[2,3]{Michalis~Vazirgiannis}

\cortext[cor1]{Corresponding author: \texttt{research@deezer.com}}
\address[1]{Deezer Research, Paris, France}
\address[2]{LIX, \'{E}cole Polytechnique, Institut Polytechnique de Paris, Palaiseau, France}
\address[3]{Athens University of Economics and Business (AUEB), Athens, Greece}

\begin{abstract}
Graph autoencoders (GAE) and variational graph autoencoders (VGAE) emerged as powerful methods for link prediction. Their performances are less impressive on community detection problems where, according to recent and concurring experimental evaluations, they are often outperformed by simpler alternatives such as the Louvain method. It is currently still unclear to which extent one can improve community detection with GAE and VGAE, especially in the absence of node features. It is moreover uncertain whether one could do so while simultaneously preserving good performances on link prediction. In this paper, we show that jointly addressing these two tasks with high accuracy is possible. For this purpose, we introduce and theoretically study a community-preserving message passing scheme, doping our GAE and VGAE encoders by considering both the initial graph structure and modularity-based prior communities when computing embedding spaces. We also propose novel training and optimization strategies, including the introduction of a modularity-inspired regularizer complementing the existing reconstruction losses for joint link prediction and community detection. We demonstrate the empirical effectiveness of our approach, referred to as Modularity-Aware GAE and VGAE, through in-depth experimental validation on various real-world graphs.
\end{abstract}

\begin{keyword}
Graph Autoencoders, Node Embedding, Modularity, Graph Neural Networks, Link Prediction, Community Detection
\end{keyword}

\end{frontmatter}

\section{Introduction}

\label{introduction}
\paragraph{\textbf{Context and Motivation}} Graph structures became ubiquitous in various fields ranging from social networks and web mining to biology \cite{hamilton2017representation,hamilton2020graph}, due to the proliferation of data representing entities, also known as (a.k.a.) \textit{nodes}, connected by links, a.k.a. \textit{edges}. Extracting relevant information from these nodes and edges is essential to tackle a wide range of graph-based machine learning problems \cite{hamilton2017representation,hamilton2020graph,zhang2018network,wu2019comprehensive}, such as the challenging~tasks~of:
\begin{itemize}
    \item \textit{link prediction} \cite{liben2007link,kumar2020link}, which consists in inferring the presence of new or unobserved edges between some pairs of nodes, based on observable edges in the graph;
    \item \textit{community detection} \cite{blondel2008louvain,malliaros2013clustering}, which consists in clustering nodes into similar subgroups, according to a chosen similarity metric.  
\end{itemize}

To effectively address such graph-based problems, significant research efforts were recently devoted to the development of \textit{node embedding} methods. In a nutshell, these methods aim to learn vectorial representations of nodes, in an \textit{embedding space} where node positions should reflect and summarize the initial graph structure \cite{hamilton2017representation,hamilton2020graph,kipf2020phd}. Then, they assess the probability of a missing edge between two nodes, or their likelihood of belonging to the same community, by evaluating the proximity of these nodes in the learned space \cite{wang2017mgae,kipf2016-2,choong2018learning}. Node embedding methods are at the core of promising improvements in real-world graph learning applications \cite{hamilton2020graph,kumar2020link}, often outperforming more traditional graph mining methods relying on hand-engineered indicators \cite{liben2007link,kumar2020link}. 

In particular, \textit{graph autoencoders} (GAE) and \textit{variational graph autoencoders} (VGAE) \cite{wang2017mgae,kipf2016-2,tian2014learning,wang2016structural} recently emerged as two powerful families of node embedding methods. They both rely on an \textit{encoding-decoding} strategy that, in a broad sense, consists in \textit{encoding} nodes into an embedding space from which \textit{decoding}, i.e., reconstructing the original graph should ideally be possible, by
leveraging either a deterministic (for GAE) or a probabilistic (for VGAE) approach. Originally mainly designed for link prediction (at least in their modern formulation leveraging \textit{graph neural networks} architectures \cite{kipf2016-2}), the overall effectiveness of GAE and VGAE and of their extensions on this specific task has been widely experimentally confirmed over the past few years \cite{pan2018arga,tran2018multi,salha2019-1,salha2019-2,huang2019rwr,grover2019graphite,semiimplicit2019,aaai20}.

On the other hand, several concurring studies \cite{choong2018learning,salha2019-1,choong2020optimizing,salha2021fastgae} have simultaneously pointed out the limitations of these models on community detection. They emphasized that standard GAE and VGAE are often outperformed by simpler node clustering alternatives, such as the popular Louvain method \cite{blondel2008louvain}. While some recent studies worked on this issue (see Section~\ref{s2} for an overview), their solutions strongly relied on \textit{clustering-oriented probabilistic} priors that only fit the VGAE setting and can not be directly transposed to GAE. They also benefited greatly from the presence of \textit{node features} a.k.a. \textit{attributes} complementing the graph structure, but provided only little to no empirical gain on featureless graphs that are nonetheless ubiquitous. 
Thirdly, they did not explicitly try to preserve the good performances of GAE and VGAE on link prediction. 
In practice, as we will argue in this work, learning node embedding spaces that jointly enable good link prediction and community detection performances is often desirable, both for real-world applications and in pursuit of learning accurate and general representations of a graph structure. 

\paragraph{\textbf{Research Questions}} In summary, the question of how to improve community detection with GAE and VGAE remains incompletely addressed, especially in the absence of node features, and it is still unclear to which extent one can improve community detection with these models without simultaneously deteriorating link prediction. In this paper, we propose to tackle these important problems by investigating the following two research questions:
\begin{itemize}
    \item \textbf{Question 1:} Can we improve community detection for \textit{both} the GAE \textit{and} VGAE settings? And does this improvement persist for \textit{featureless} graphs?
    \item \textbf{Question 2:} Do improvements in the community detection task necessarily incur a loss in the link prediction performance or can they be \textit{jointly} addressed with high accuracy? 
\end{itemize}


\paragraph{\textbf{Contributions}} In this paper, we propose several novel contributions to both the GAE and VGAE frameworks, which allow us to answer both of these research questions positively. More precisely, our contributions are listed as follows. 

\begin{enumerate}
    \item We first diagnose the reasons why GAE and VGAE models tend to perform well on link prediction but to underperform on community detection.
    \item Then, based on insights from this diagnosis, we improve GAE and VGAE for community detection while preserving their ability to identify missing edges. Our strategy leverages concepts inspired by \textit{modularity-based} clustering \cite{blondel2008louvain,brandes2007modularity,shiokawa2013fast}:
    \begin{enumerate}
        \item Specifically, we first present and theoretically study a novel \textit{community-preserving message passing scheme}, doping our GAE and VGAE encoders by considering both the initial graph structure and modularity-based prior communities when computing embedding spaces;
        \item We also introduce revised training and optimization strategies with respect to (w.r.t.) current practices in the scientific literature, including the introduction of modularity-inspired losses complementing the existing reconstruction losses with the aim of jointly ensuring good performances on link prediction and community detection.
    \end{enumerate}
    
    \item Backed by in-depth experiments on several real-world graphs, including on industrial-scale data provided by a global music streaming service, we demonstrate the empirical effectiveness of our approach at addressing:
    \begin{enumerate}
        \item Pure community detection problems;
        \item Joint community detection and link prediction problems.
    \end{enumerate}
    \item Lastly, along with this paper, we publicly release our source code on GitHub, to ensure the reproducibility of our results and to encourage future usage~of~our~method.
\end{enumerate}     

\paragraph{\textbf{Organization of this Paper}} The remainder of this paper is organized as follows. In Section \ref{s2}, we recall key concepts related to GAE and VGAE models, as well as their existing applications to link prediction and to community detection. We also point out the limits of current GAE and VGAE models on community detection. In Section \ref{s3}, we diagnose the reasons explaining these limits. We subsequently introduce and theoretically study our proposed solution, referred to as \textit{Modularity-Aware GAE and VGAE}, to overcome these limitations. We report and discuss our experimental evaluation in Section \ref{s4}, and we conclude in Section \ref{s5}.

\section{Background and Related Work}
\label{s2}

We begin this section by providing a description of the GAE and VGAE models in Sections \ref{sec:GAE} and \ref{sec:VGAE}, respectively. We then introduce the two learning tasks we address in this work and previous solution approaches. Specifically, we discuss the link prediction task in Section \ref{sec:link_prediction} and the community detection task in Section \ref{sec:comm_detection}.

Throughout this paper, we consider an undirected graph $\mathcal{G} = (\mathcal{V},\mathcal{E})$ with $|\mathcal{V}| = n$ nodes and $|\mathcal{E}| = m$ edges. We denote by $A$ the $n\times n$ adjacency matrix of $\mathcal{G}$, that is either:
\begin{itemize}
    \item binary, i.e., for all $(i,j) \in \mathcal{V} \times \mathcal{V},$ we have $ A_{ij} \in \{0,1\}$; 
    \item or weighted and normalized,  i.e.,  for all $ (i,j) \in \mathcal{V} \times \mathcal{V},$ we have $ A_{ij} \in [0,1].$ 
\end{itemize}Each node $i \in \mathcal{V}$ is also equipped with an $f$-dimensional feature vector $x_i$. In the following, $X$ denotes the $n \times f$ feature matrix stacking up all feature vectors. When dealing with featureless graphs, we will simply set $X = I_n$, where $I_n$ denotes the $n \times n$ identity matrix. In Table \ref{tab:notation} we provide an overview of our frequently used notation.

\begin{table}[t]
\centering
\caption{Overview of frequently used notation.}
 \vspace{0.1cm}
\label{tab:notation}
{\small
\begin{tabularx}{\textwidth}{llX}
\toprule
\textbf{Notation} & \textbf{Domain}& \textbf{Description} \\
\midrule
\midrule
$\mathcal{G} = (\mathcal{V},\mathcal{E})$ & -- & A graph composed of a node set $\mathcal{V}$ and an edge set $\mathcal{E}$\\
$n, m$ & $\mathbb{N}^{+}$ & Number of nodes and edges in $\mathcal{G},$ respectively\\
$f$ & $\mathbb{N}^{+}$ & Dimension of feature vectors describing nodes \\
$x_i$ & $\mathbb{R}^{f}$ & Feature vector describing node $i \in \mathcal{V}$  \\
$X=[x_1, \ldots, x_n]^T $ &$\mathbb{R}^{n\times f}$ & Node feature matrix, stacking up $x_i$ vectors of all nodes\\
$A$ & $[0,1]^{n \times n}$& Adjacency matrix of $\mathcal{G}$\\
$D=\mathrm{diag}(A \mathbf{1}_n)  $ & $\mathbb{R}^{n \times n}$& Degree matrix corresponding to the adjacency matrix $A$\\
$\agcn{A}$ & $[0,1]^{n \times n}$ & Symmetric normalization of $A$ defined in Equation~\eqref{eq:norm}, and used as message passing operator in Graph Convolutional Networks (GCN)\\
$d$ & $\mathbb{N}^+$ & Dimension of the node embedding space learned by GAE and VGAE models \\ 
$z_i$ & $\mathbb{R}^{d}$ & Node embedding vector of node $i \in \mathcal{V}$\\
$Z=[z_1, \ldots, z_n]^T$ & $\mathbb{R}^{n \times d}$& Node embedding matrix, stacking up $z_i$ vectors of all nodes\\
$\hat{A}$ & $ [0,1]^{n \times n}$&  Adjacency matrix reconstructed by the GAE or VGAE\\
$\lossGAE$ & $ \mathbb{R}^{+}$& Reconstruction loss minimized in standard GAE \\
$\lossVGAE$ &$\mathbb{R}^{-}$& Evidence lower bound maximized in standard VGAE\\
$\lossMAGAE$ & $\mathbb{R}$\ &Revised loss minimized in our Modularity-Aware GAE \\
$\lossMAVGAE$ &$\mathbb{R}$ & Revised objective maximized in our Modularity-Aware VGAE\\
$C_1, \ldots, C_\ncluster$ &$\mathcal{V}$& A partition of $\mathcal{V}$ into $\ncluster$ disjoint node clusters/communities\\
$n_1, \ldots, n_\ncluster$ & $\{1,...,n\}$ & Number of nodes in $C_1, \ldots, C_\ncluster$ respectively\\
$\Ac$ & $\{0,1\}^{n \times n}$ & Community membership matrix defined in Equation~\eqref{eqAc}, corresponding to the 
adjacency matrix of a graph defined by the $C_1, \ldots, C_\ncluster$ partition\\
$\Ao$ & $\{0,1\}^{n \times n}$ & $\dregc$-regular sparsification of $\Ac$, in which each node in $C_l$ is only connected to $\dregc < n_l$ randomly selected  nodes in $C_l$\\
$\lambda$ &$\mathbb{R}^{+}$& Hyperparameter regulating the relative importance of $\Ao$ in encoders of the Modularity-Aware GAE and VGAE\\
$\beta, \gamma$ &$\mathbb{R}^{+}$& Hyperparameters regulating the modularity component in the revised losses of the Modularity-Aware GAE and VGAE\\
\bottomrule
\end{tabularx}
}
\end{table}

\subsection{\textbf{Graph Autoencoders}} \label{sec:GAE}

The term \textit{graph autoencoders} (GAE) refers to a family of unsupervised models learning node embedding spaces from graph data~\cite{kipf2020phd,kipf2016-2,tian2014learning,wang2016structural,salha2020simple}. They involve the combination of two components, an \textit{encoder} and a \textit{decoder}, that jointly learn low-dimensional vectorial representations of nodes from which one should be able to reconstruct the initial graph. Intuitively, the ability to accurately reconstruct a graph from a node embedding space indicates that this space preserves some important information about the graph~structure. 

Albeit under various formulations, this encoding-decoding strategy has been widely adopted over the last years to learn node embedding spaces in the absence of node labels \cite{kipf2020phd,wang2017mgae,kipf2016-2,tian2014learning,wang2016structural}. In this section, we choose to mainly follow the formulation of Kipf and Welling \cite{kipf2016-2}, as their work is explicitly mentioned as the seminal reference in the majority of the most recent advances in GAE, including \cite{choong2018learning,pan2018arga,salha2019-1,salha2019-2,huang2019rwr,grover2019graphite,semiimplicit2019,aaai20,choong2020optimizing,salha2020simple,pei2021generalization,li2020dirichlet}.

\subsubsection{Encoder}

The first component of a GAE model is an \textit{encoder}. In its most general formulation, it is a parameterized function processing $A$ and $X$, and aiming to map each node $i \in \mathcal{V}$ from $\mathcal{G}$ to a low-dimensional \textit{embedding vector} $z_i \in \mathbb{R}^d$, with $d \ll n$. Denoting by $Z$ the $n \times d$ matrix stacking up all vectors $z_i$, we have:
\begin{equation}
Z = \text{Encoder}(A,X).    
\end{equation}

In practice, a \textit{graph neural network} architecture \cite{hamilton2020graph,zhang2018network} often acts as the encoder. In particular, Kipf and Welling \cite{kipf2016-2} leverage multi-layer \textit{graph convolutional network} (GCN) encoders \cite{kipf2016-1}. In a $L$-layer GCN, with $L \geq 2$, with input layer $H^{(0)} =X$, and with output
layer $H^{(L)} = Z$, we have:
\begin{equation} 
\begin{cases}
H^{(0)} = X, \\
H^{(l)} = \text{ReLU} (\agcn{A} H^{(l-1)} W^{(l-1)}), \hspace{5pt} \text{for } l \in \{1,...,L-1\}, \\
H^{(L)} = Z = \agcn{A} H^{(L-1)} W^{(L-1)},
\end{cases}
\label{eq:gcn}
\end{equation}
where 
\begin{equation}
\agcn{A} = (D+I_n)^{-\frac{1}{2}}(A+I_n)(D+I_n)^{-\frac{1}{2}
\label{eq:norm}
}
\end{equation} is the \textit{symmetric normalization} of the adjacency matrix~$A$, and where $D=\mathrm{diag}(A \mathbf{1}_n)$ ($\mathbf{1}_n$ denotes the vector containing $n$ entries all equal to $1$) is the diagonal degree matrix of $A.$ At each layer $l>1$, the GCN computes a vectorial representation for each node $i\in \mathcal{V}$, by computing a weighted average of the representations from layer $l-1$ of $i$'s direct neighbors and of $i$ itself. This averaging operation is composed with a linear transformation via trainable weight matrices $W^{(0)},\ldots,W^{(L-1)}$ and a ReLU activation function: $\text{ReLU}(x) = \max(x,0)$. The tuning of these \textit{weight matrices}, as proposed in \cite{kipf2016-2}, will be detailed in Section~\ref{s213}.



To this day, GCNs remain popular encoders in GAE extensions \cite{choong2018learning,pan2018arga,salha2019-2,huang2019rwr,grover2019graphite,semiimplicit2019,aaai20,pei2021generalization,salha2021fastgae} building upon Kipf and Welling \cite{kipf2016-2}.
This can be explained by the recent successes of GCN-based models \cite{hamilton2020graph,zhang2018network,kipf2016-1}, as well as by the  simplicity of GCNs in comparison to other graph neural networks \cite{zhang2018network,bruna2013spectral} and their linear time complexity w.r.t. the number of edges $m$ in the graph \cite{kipf2016-1}. Nonetheless, the choice of GCN encoders is made without loss of generality, as they can be replaced by alternatives, including by faster \cite{chen2018fastgcn,chiang2019cluster,hamilton2017inductive}, by more sophisticated \cite{bruna2013spectral,defferrard2016,velivckovic2019graph} or, on the contrary, by simpler \cite{choong2020optimizing,salha2020simple} models.

\subsubsection{Decoder}

The second component of a GAE is a \textit{decoder}. This function aims to reconstruct an $n \times n$ adjacency matrix $\hat{A}$, estimated from the embedding vectors:
\begin{equation}
\hat{A} = \text{Decoder}(Z).    
\end{equation}
While another neural network could act as a decoder \cite{wang2016structural,park2019symmetric,li2020graph}, Kipf and Welling \cite{kipf2016-2} and most of the aforementioned extensions rely on simpler \textit{inner~product decoders}:
\begin{equation}
\hat{A} = \sigma(ZZ^T),
\end{equation}
where $\sigma(\cdot)$ denotes the sigmoid function $\sigma(x) = 1/(1 + e^{-x})$. Therefore, for all node pairs $(i,j) \in \mathcal{V}\times\mathcal{V}$, we have $\hat{A}_{ij} = \sigma (z^T_i z_j) \in [0, 1]$. In such a setting, a large and positive inner~product $z^T_i z_j$ in the node embedding space indicates the likely presence of an edge between nodes $i$ and $j$ in $\mathcal{G}$, according to the model. Again, the choice of inner~product decoders is made without loss of generality, and recent efforts considered replacing them by alternatives verifying some desirable properties, such as the ability to reconstruct directed edges \cite{salha2019-2}, to capture triads structures \cite{aaai20} or to reconstruct biologically plausible graphs in the case of autoencoders for molecular structures \cite{molecule1,simonovsky2018graphvae}.

\subsubsection{Optimization}
\label{s213}

Recall that GAE models aim to learn node embedding spaces from which one can accurately reconstruct graphs. They are trained to minimize \textit{reconstruction losses}, that evaluate the similarity between the decoded adjacency matrix $\hat{A}$ and the original one $A$. Specifically,
Kipf and Welling \cite{kipf2016-2} train weight matrices of their GCN encoders by iteratively minimizing, using gradient descent \cite{goodfellow2016deep}, the following \textit{cross-entropy} loss:
\begin{align}
\lossGAE = \frac{-1}{n^2}\sum_{(i,j) \in \mathcal{V}\times \mathcal{V}} \Big[A_{ij}\log(\hat{A}_{ij}) + (1-A_{ij})\log(1 - \hat{A}_{ij})\Big].
\label{lossGAE}
\end{align}
In the case of sparse graphs where unconnected node pairs significantly outnumber the connected ones, i.e., the graph's edges, it is common to reweight the ``positive terms'' in Equation (\ref{lossGAE}) by a factor $w_{\text{pos}} > 1$ \cite{kipf2016-2,salha2021fastgae}, or alternatively to subsample ``negative terms'' \cite{kipf2020phd,pytorchgeometric}. We also note that an exact evaluation of $\lossGAE$ requires the reconstruction of the entire matrix $\hat{A}$, which suffers from a quadratic $O(dn^2)$ time complexity. For scalability concerns, recent works proposed faster training strategies \cite{salha2019-1,salha2021fastgae,pytorchgeometric}. This includes the FastGAE method \cite{salha2021fastgae} that approximates $\lossGAE$ by reconstructing stochastic \textit{subgraphs} of $O(n)$ size, and that we will also use in Section~\ref{s4} in our experiments on large graphs.

\subsection{\textbf{Variational Graph Autoencoders}} \label{sec:VGAE}

Kipf and Welling \cite{kipf2016-2} also considered a probabilistic variant of GAEs, extending \textit{variational autoencoders} (VAE) from Kingma and Welling \cite{kingma2013vae}. Besides constituting generative models with promising recent applications to graph generation \cite{molecule1,simonovsky2018graphvae,molecule3}, variants of \textit{variational graph autoencoders} (VGAE) also turned out to be effective alternatives to GAE in some link prediction or community detection tasks \cite{kipf2016-2,salha2019-2,semiimplicit2019,choong2020optimizing,salha2021fastgae,salha2020simple}. Consequently, we see value in considering both GAE and VGAE in our work.

\subsubsection{Encoder}

VGAE models provide an alternative strategy to learn a matrix $Z,$ stacking up one embedding vector $z_i$ for each node $i$ of a graph $\mathcal{G}$, by assuming that these vectors are drawn from specific distributions. In particular, Kipf and Welling \cite{kipf2016-2} assume that each vector $z_i$ is a sample drawn from a $d$-dimensional Gaussian distribution, with mean vector $\mu_i \in \mathbb{R}^d$ and variance matrix $\text{diag}(\sigma_i^2)  \in \mathbb{R}^{d \times d}$ (with $\sigma_i \in \mathbb{R}^d$). They rely on \textit{two encoders} to learn these parameters. Denoting the $n \times d$ matrices stacking up the $d$-dimensional mean
and (log)-variance vectors for each node by $\mu$ and by $\log \sigma$, respectively, they set:
\begin{equation} \mu = \text{Encoder}_{\mu}(A,X) \text{ and } \log \sigma = \text{Encoder}_{\sigma}(A,X).
\label{vaeencoder}
\end{equation}
As is the case for GAE, multi-layer GCNs often act as encoders, i.e., $\mu = \text{GCN}_{\mu}(A,X)$ and $\log \sigma = \text{GCN}_{\sigma}(A,X)$. Then, they adopt a \textit{mean-field inference model} for $Z$  \cite{kipf2016-2}, i.e.,
\begin{align} \label{eqn:distnZ}
q(Z \mid A,X) = \prod_{i=1}^n q(z_i\mid A,X), \text{ with } q(z_i \mid A,X) = \mathcal{N}(z_i\mid \mu_i, \text{diag}(\sigma_i^2)),
\end{align}
where $\mathcal{N}(\cdot)$ denotes the normal distribution.

\subsubsection{Decoder}
In the VGAE setting, the actual embedding vectors $z_i$ are sampled  from the aforementioned normal distributions. From such embedding representations, VGAE models then require a \textit{generative model} $p(A \mid Z,X)$, to act as a graph \textit{decoder}. As for GAE, Kipf and Welling \cite{kipf2016-2} rely on inner~products together with sigmoid activation functions to reconstruct edges: 
\begin{align} \hat{A}_{ij} = p(A_{ij} = 1 \mid z_i, z_j) = \sigma(z_i^Tz_j),
\end{align}
where the embeddings $z_i, z_j$ are sampled from the distribution in Equation \eqnref{eqn:distnZ}.
Then, the authors assume the following generative model which factorizes over the edges:
\begin{align}p(A \mid Z,X) = \prod\limits_{i=1}^n  \prod\limits_{j=1}^n  p(A_{ij} \mid z_i, z_j).\end{align}

\subsubsection{Optimization}
\label{s223}
During training, and similarly to standard VAE models \cite{kingma2013vae}, Kipf and Welling \cite{kipf2016-2} iteratively maximize a tractable variational lower bound of the model's likelihood, a.k.a. the \textit{evidence lower bound} (ELBO) \cite{kingma2013vae}, written as follows in the context of VGAE:
\begin{align}\lossVGAE = \mathbb{E}_{q(Z \mid A,X)} \Big[\log
p(A \mid Z,X)\Big] - \mathcal{D}_{KL}\Big(q(Z \mid A,X)||p(Z)\Big).
\label{elbo}
\end{align}
This ELBO is iteratively maximized w.r.t. weights of the two GCN encoders, by gradient descent and potentially using approximation strategies such as the strategies described in Section~\ref{s213}. In the above Equation (\ref{elbo}), $\mathcal{D}_{KL}(\cdot||\cdot)$ denotes the Kullback-Leibler divergence~\cite{kullback1951information}, and $p(Z)$ corresponds to a unit Gaussian 
prior on the distribution of the latent vectors, that can also be interpreted as a regularization term on the magnitude of the embedding vectors. We refer to \cite{kipf2020phd,kingma2013vae,doersch2016tutorial} for complete details and derivations of such ELBO bounds in the context of VAE models.

\subsection{\textbf{Evaluating GAE and VGAE: Link Prediction}} \label{sec:link_prediction} 

The question of how to properly determine the quality of node embedding representations learned from GAE and VGAE models is crucial. While one could directly report reconstruction losses \cite{wang2016structural}, recent research work instead strives to apply the GAE and VGAE models to \textit{downstream evaluation tasks}, which permit reporting more insightful and interpretable evaluation metrics \cite{wang2016structural,tian2014learning,kipf2016-2}. In particular, Kipf and Welling \cite{kipf2016-2} evaluate their GAE and VGAE models on \textit{link~prediction} problems in citation networks~\cite{sen2008collective}.

\subsubsection{The Link Prediction Task}
\label{s231}

Kipf and Welling \cite{kipf2016-2} follow an evaluation methodology consisting of:
\begin{itemize}
    \item Training their models on an \textit{incomplete} version of an original graph, for which only a certain percentage of randomly sampled edges (85\% in their case) are visible;
    \item Constructing \textit{validation} and \textit{test sets} gathering:
    \begin{itemize}
        \item node pairs corresponding to missing edges ($5\%$ and $10\%,$ respectively, in \cite{kipf2016-2});
        \item the same number of randomly picked unconnected node pairs in the graph; 
    \end{itemize} 
    \item Evaluating the models' abilities to distinguish edges from non-edges in~these~sets, using node embedding representations learned on the incomplete graph.
\end{itemize}  
Indeed, while all node pairs in the validation and test sets are observed to be unconnected during training, half of them actually correspond to missing edges from the original graph. Link prediction acts as a binary classification downstream task, evaluating to which extent the decoder's predictions $\hat{A}_{ij} = \sigma(z^T_i z_j)$ correctly locate and reconstruct these missing edges despite their absence during training. Performance in the link prediction task is evaluated using metrics such as the \textit{Area Under the ROC Curve} (AUC) and \textit{Average Precision} (AP) scores \cite{pedregosa2011scikit}.

\subsubsection{Link Prediction with GAE, VGAE and Extensions}

Kipf and Welling \cite{kipf2016-2} show that their proposed GAE and VGAE reach competitive link prediction scores w.r.t. some popular node embedding methods, such as DeepWalk \cite{perozzi2014deepwalk} and Laplacian eigenmaps \cite{von2007tutorial}. They also emphasize an additional benefit of the GCN-based GAE and VGAE over baseline methods such as DeepWalk and Laplacian eigenmaps, which is the ability to leverage both the graph structure and node features when learning embedding spaces. 

Over the last few years, the overall effectiveness of the GAE and VGAE paradigms at addressing link prediction has been widely confirmed experimentally \cite{pan2018arga,tran2018multi,salha2019-1,salha2019-2,huang2019rwr,grover2019graphite,semiimplicit2019,aaai20,salha2021fastgae,salha2020simple,pei2021generalization,hao2020inductive,rennard2020graph,berg2018matrixcomp}. Numerous research efforts proposed and evaluated variants of GAE and VGAE designed for this specific task, improving their performances by considering more refined encoders \cite{salha2019-1,semiimplicit2019,hao2020inductive,wu2021deepened}, decoders \cite{salha2019-2,grover2019graphite,semiimplicit2019,aaai20} or regularization techniques \cite{pan2018arga,huang2019rwr,pei2021generalization}. Other works also successfully addressed different downstream tasks that are closely related to link prediction, such as edge classification \cite{rennard2020graph} or graph-based recommendation \cite{hao2020inductive,berg2018matrixcomp,salha2021cold}.

\subsection{\textbf{Evaluating GAE and VGAE: Community Detection}} \label{sec:comm_detection}

While link prediction remains a prominent evaluation task for GAE and VGAE, they have also shown promising results on (semi-supervised) node classification \cite{tran2018multi,semiimplicit2019}, canonical correlation analysis \cite{kaloga2021multiview} and, in the case of VGAE, graph generation especially in the context of molecular graph data \cite{molecule1,simonovsky2018graphvae,molecule2}. However, their performances are less impressive on \textit{community detection} \cite{choong2018learning,choong2020optimizing}, on which we focus in this section.
\subsubsection{The Community Detection Task}
Among the fundamental problems in graph-based machine learning, \textit{community detection} (which we regard as a synonym of \textit{node clustering} in this work, consistently with Fortunato~\cite{fortunato2010community})  consists in identifying $\ncluster < n$ clusters a.k.a. \textit{communities} of nodes that, in some sense, are more similar to each other than to the other nodes \cite{malliaros2013clustering,choong2018learning,fortunato2010community}.  More formally, we aim to obtain a partition of the node set $\mathcal{V}$ into $\ncluster$ sets:
\begin{equation}
C_1 \subseteq \mathcal{V}, \ldots, C_\ncluster \subseteq \mathcal{V},
\end{equation}
with cardinality $|C_\clustersubscript| = n_\clustersubscript \leq n $ for $\clustersubscript \in \{1, \ldots, \ncluster\}$. 
The quality of such a partition is usually assessed through some predefined similarity metrics, e.g., unsupervised density-based metrics\footnote{Such metrics usually rely on \textit{homophily} assumptions. This term describes the tendency of nodes to connect to ``similar'' nodes in the graph, which is observed in numerous real-world applications \cite{kumar2020link}. Under such assumptions, intuitively, nodes from the same community should be more densely connected, and nodes from different communities should be more sparsely connected.} calculated from the intra- and inter-cluster edge density
\citep{malliaros2013clustering}, or scores such as the normalized \textit{Mutual Information} (MI) \cite{choong2018learning} that compares the partition to some ground-truth node labels hidden during training. 

Improving community detection on graphs has been the objective of significant efforts over the last decades (see, e.g., \cite{malliaros2013clustering,fortunato2010community} for a review), and still constitutes an active area of research \cite{choong2018learning,cavallari2017learning,tu2018unified,sun2019vgraph,he2021community} with numerous applications.
This includes the segmentation of websites in a web graph according to thematic categories, as well as the detection of densely connected subgroups of users in online social networks \cite{malliaros2013clustering}.

\subsubsection{Community Detection with GAE and VGAE}
\label{s242}
In the presence of node embedding representations, community detection boils down to the more standard problem of clustering $n$ vectors in a $d$-dimensional Euclidean space into $\ncluster$ groups \cite{macqueen1967some}. With this goal in mind, several studies specifically tried to perform community detection with GAE and VGAE by:
\begin{itemize}
    \item learning an embedding vector $z_i$ for each $i \in \mathcal{V}$, as described in Sections \ref{sec:GAE} and \ref{sec:VGAE};
    \item clustering the resulting vectors $z_i$ into $\ncluster$ groups, through one of the numerous clustering methods for Euclidean data, such as the popular $k$-means algorithm \cite{macqueen1967some}.
\end{itemize} 

However, concurring experimental evaluations \cite{choong2018learning,salha2019-1,choong2020optimizing,salha2021fastgae}  recently pointed out the limitations of such an approach. They emphasized its lower performance w.r.t. simpler community detection alternatives, that sometimes even directly operate on the graph structure without considering node features, such as the popular Louvain method~\cite{blondel2008louvain}. 

For instance, Choong et al. \cite{choong2020optimizing} show that, on the (featureless) Cora citation network~\cite{sen2008collective}, a VGAE+$k$-means strategy reaches a mean normalized MI score of 23.84\%,   way below the Louvain method (43.36\%). Salha et al. \cite{salha2021fastgae} show that, on the same graph, a GAE+$k$-means also reaches an underwhelming 30.88\% mean normalized MI score. These authors obtain comparable conclusions on several other popular graph datasets, such as the featureless versions \cite{sen2008collective} of Citeseer (9.85\% MI for VGAE+$k$-means vs 16.39\% for Louvain, in \cite{salha2021fastgae}) and Pubmed (20.41\% MI for VGAE+$k$-means in \cite{choong2020optimizing}, which is comparable to Louvain, but significantly below the 29.46\% MI score obtained by running a $k$-means on node embedding vectors learned via DeepWalk \cite{perozzi2014deepwalk}).

\subsubsection{Community Detection with Extensions of GAE and VGAE} \label{sec:CD_GAE_extensions}

Several studies have worked on the issue of the underwhelming performance of GAE and VGAE in the community detection task \cite{choong2018learning,choong2020optimizing,li2020dirichlet}. Choong~et~al.~\cite{choong2018learning} introduced VGAECD, a \textit{VGAE for Community Detection (CD)} model that replaces Gaussian priors by learnable \textit{Gaussian mixtures}. Such a choice permits recovering communities from node embedding spaces without relying on an additional $k$-means step. In a subsequent study \cite{choong2020optimizing}, the same authors proposed VGAECD-OPT, an improved version of VGAECD. Specifically, VGAECD-OPT replaces GCN encoders with simpler linear models \cite{wu2019simplifying}, as proposed in Salha et al. \cite{salha2020simple}. It also adopts a different optimization procedure based on neural expectation-maximization \cite{greff2017neural}, which guarantees that communities do not collapse during training \cite{choong2020optimizing} and experimentally leads to better performances.

More recently, Li et al. \cite{li2020dirichlet} introduced \textit{Dirichlet Graph Variational Autoencoder} (DGVAE), another extension of VGAE which uses Dirichlet distributions as priors on latent vectors, acting as indicators of community memberships. The \textit{Marginalized GAE} (MGAE) model from \citet{wang2017mgae} is also evaluated on community detection. However, the MGAE model does not explicitly leverage embedding representations for this task; instead the \textit{spectral clustering} \cite{von2007tutorial} is applied to the decoded~graphs.
Lastly, while community detection was not the main focus in \cite{pan2018arga,huang2019rwr,aaai20,park2019symmetric,pei2021generalization}, these works all proposed various encoding-decoding methods that, to different extents, seem to outperform standard GAE and VGAE on the community detection task, in the reported evaluations. They consider alternatives encoder or training choices, which we further discuss and investigate~in~Section~\ref{s3}.

\subsubsection{Limitations} \label{sec:limits}

While the models discussed in Section \ref{sec:CD_GAE_extensions} will constitute relevant baselines in our experiments (see Section~\ref{s4}), they still suffer from several fundamental limitations that motivate our work.
\begin{itemize}
    \item Firstly, all extensions explicitly designed for community detection \cite{choong2018learning,choong2020optimizing,li2020dirichlet} \textit{rely on clustering-oriented probabilistic priors}. They are only applicable in the VGAE paradigm, and cannot be directly transposed to the deterministic GAE setting. The question of how to design clustering-efficient GAE models thus remains widely open.
    
    \item More importantly, a closer look at these models reveals that their \textit{empirical gains often mostly stem from the addition of node features} to the graph. As an illustration, Table~\ref{tab:vgaecd} displays the reported performances of VGAECD and VGAECD-OPT on the \textit{featureless} versions of two graphs \cite{choong2020optimizing}. We observe that they offer little to no empirical advantage when features are absent. This draws into question the extent to which these models are able to capture communities from~(only)~graph~data.

  \begin{table}
      \caption{Normalized mutual information scores (in \%) for community detection on the Cora and Pubmed citation networks, \textit{with} and \textit{without} node features. Results are directly taken from the evaluation of Choong et al.~\cite{choong2020optimizing}. This table emphasizes that, in the absence of node features, VGAECD and VGAECD-OPT bring little (to no) advantage w.r.t. standard VGAE, and remain below the Deepwalk and/or Louvain baselines. Scores of VGAECD and VGAECD-OPT significantly increase when adding features to the graph. Recall: in this table, Deepwalk and Louvain both ignore node features.}
       \vspace{0.2cm}
    \label{tab:vgaecd}
   \resizebox{\textwidth}{!}{
  \begin{tabular}{c|c|cc|ccc|cc}
     \textbf{Dataset} & \textbf{VGAE} & \textbf{VGAECD} & \textbf{VGAECD-OPT} & \textbf{DeepWalk} & \textbf{Louvain} \\
    \midrule \midrule
    Cora \textit{without} node features & 23.84 & 28.22 & 37.35 & 37.96 & \textbf{43.36} \\

    Pubmed \textit{without} node features & 20.41 & 16.42 & 25.05 & \textbf{29.46} & 19.83 \\
\midrule
    Cora \textit{with} node features& 31.73 & 50.72 & \textbf{54.37} & 37.96 & 43.36\\
    Pubmed \textit{with} node features & 19.81 & 32.53 & \textbf{35.52} & 29.46 & 19.83\\
    
    \bottomrule
  \end{tabular}}

    \vspace{-0.1cm}
\end{table}

    The important role of node features has subsequently been confirmed (e.g., Park et al.~\cite{park2019symmetric} show that, on the Pubmed dataset, a straightforward $k$-means on the node features alone reaches comparable MI scores w.r.t. VGAE and MGAE). On the other hand, most of the aforementioned other studies with empirical improvements \cite{pan2018arga,huang2019rwr,aaai20,park2019symmetric,pei2021generalization} only reported results on graphs equipped with node features. This motivates the need for a proper investigation of the \textit{featureless} case where models cannot rely on the additional node feature information. 

    \item Lastly, previous studies centered around community detection \cite{wang2017mgae,choong2018learning,choong2020optimizing,li2020dirichlet} \textit{did not explicitly try to preserve the good performances of GAE and VGAE on link prediction}. Overall, most of the aforementioned existing works learn node representations specific to a particular learning task. Therefore, it is still unclear whether one can improve community detection with GAE or VGAE without simultaneously deteriorating link prediction.
    
    With the general aim of learning high-quality node embeddings, one can wonder to which extent these models can learn representations that are \textit{jointly} useful for several tasks. Besides providing a more accurate summary of the graph structure under consideration, such representations could also lead to significant resource savings in real-world applications. As an illustration, our experiments in Section~\ref{s4} will consider industrial-scale graph data obtained from the music streaming service Deezer, providing a concrete example of an application for which learning representations jointly effective at link prediction and community detection is desirable.
\end{itemize}
In conclusion to this section, the question of how to effectively improve community detection with GAE and VGAE remains incompletely addressed.

\section{Modularity-Aware (Variational) Graph Autoencoders}
\label{s3}

\begin{figure}[t]
    \centering
    \includegraphics[width=\textwidth]{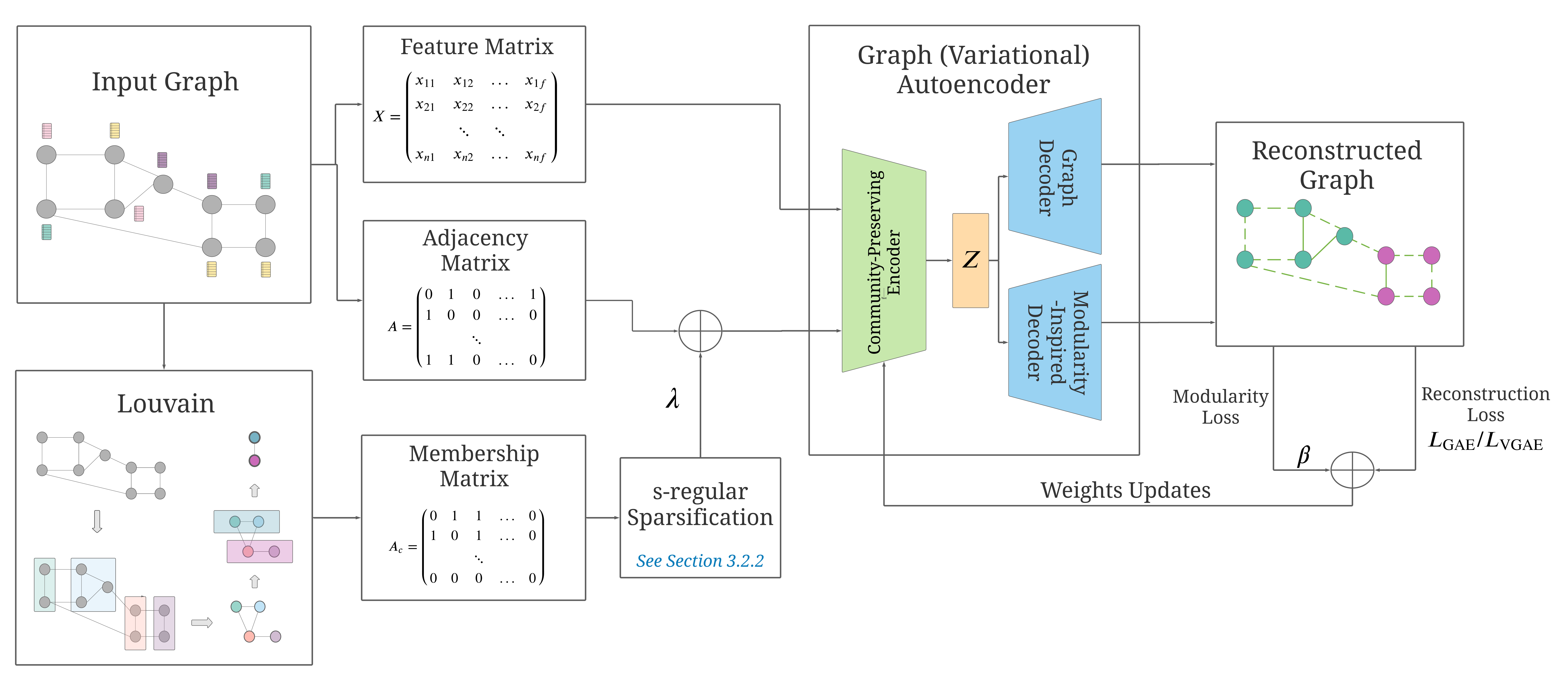} 
    \caption{Overview of our proposed \textit{Modularity-Aware GAE/VGAE} model. Firstly, input graph data $A$ and $X$ are combined with the $\dregc$-regular sparsified prior community membership matrix $\Ao$, derived through iterative modularity maximization via the Louvain algorithm, as described in Section~\ref{s32}. Then, they are processed by our revised community-preserving (Linear or GCN) encoders, encoding each node $i$ as an embedding vector $z_i$ of dimension $d \ll n$. Neural weights of encoders are optimized through a procedure combining reconstruction and modularity-inspired losses, and described in Section~\ref{s331}. Furthermore, other hyperparameters from this model are tuned via the method described in Section~\ref{s332} and designed for joint link prediction and community detection applications.}
    \label{fig:workflow}
\end{figure}

We now introduce our approach, referred to as \textit{Modularity-Aware GAE and VGAE} in the following, to address the aforementioned limitations.  
In Section~\ref{s31}, we first provide a general overview of the key components of our solution. They transpose concepts from \textit{modularity-based} clustering \cite{blondel2008louvain,brandes2007modularity,shiokawa2013fast} to GAEs and VGAEs, and are illustrated in Figure~\ref{fig:workflow}. 
We subsequently detail these solution components in Sections~\ref{s32}~and~\ref{s33}. 

\subsection{\textbf{Diagnosis and Overview of our Proposed Solution}}
\label{s31}

Based on our literature review, we diagnose three main reasons that can explain why previous GAE and VGAE models still suffer from the limitations described in Section~\ref{sec:limits}.
\begin{itemize}
    \item Firstly, they leveraged \textit{encoders} that were not specifically designed to preserve the intrinsic communities from the graph structure under consideration in the node embedding space. This includes the popular GCN, as well as refined neural models that rather aimed to preserve clusters from node features (but not necessarily the actual communities from the graph under consideration).
    
    In \textit{Modularity-Aware GAE and VGAE}, we overcome this issue by incorporating a novel encoding scheme for graph community-preserving representation learning. It consists in an improvement of the GCN \textit{message passing operator}, boosting both GAE and VGAE models by simultaneously considering the initial graph structure and \textit{modularity-based node communities} when computing node embedding spaces. We present and theoretically study this encoder in Section~\ref{s32}.
    
    \item Besides the encoder's architecture, previous models were often \textit{optimized} in a fashion that, by design, favors link prediction over community detection. In particular, the standard cross-entropy (Equation \eqref{lossGAE}) and ELBO (Equation \eqref{elbo}) losses, used to learn neural weight matrices, directly involve the reconstruction of \textit{node pairs} from the embedding space\footnote{In the case of the probabilistic VGAE paradigm, another limitation of the ELBO loss - and of the underlying generative decoder - lies in the use of standard Gaussian priors. Replacing these priors by for example \textit{Gaussian mixtures} as in \cite{choong2018learning,choong2020optimizing}, appears to be an intuitive approach for community-based learning. However, as this approach 1) does not extend to deterministic GAE, and 2) has been extensively studied in \cite{choong2018learning,choong2020optimizing}, we do not further develop it in this work. We will nonetheless compare to \cite{choong2018learning,choong2020optimizing} in experiments, and will argue in Section~\ref{s331} that Gaussian mixtures could straightforwardly be incorporated in our proposed Modularity-Aware VGAE.}. However, as we will detail, a good reconstruction of \textit{local} pairwise connections does not necessarily imply a good reconstruction of the \textit{global}~community~structure. 
    
    In \textit{Modularity-Aware GAE and VGAE}, we instead optimize an alternative loss inspired by the \textit{modularity} \cite{blondel2008louvain}. Such a loss acts as a simple yet effective global regularization over pairwise reconstruction losses, with desirable properties for joint link prediction and community detection. It will empirically enable a refined optimization of the weight matrices from our encoders. We present this aspect in more details~in~Section~\ref{s331}.
    
    \item Lastly, in addition to these weight matrices, GAE and VGAE models involve several other hyperparameters, ranging from the number of training iterations to the learning rate \cite{kipf2016-2}. While they also impact the model  performance, the selection procedure for such hyperparameters was sometimes omitted in previous works \cite{choong2020optimizing} or based on link prediction validation sets \cite{salha2019-1,salha2021fastgae} (while, intuitively, the best hyperparameters for community detection might differ from those for link prediction).
    
    For the \textit{Modularity-Aware GAE and VGAE} we adopt an alternative graph-based model selection procedure. It completes the previous two aspects, by providing the most relevant GAE/VGAE hyperparameters for joint link prediction and community selection. We present and discuss this procedure in Section~\ref{s332}.
    \end{itemize}

\subsection{\textbf{Community-Preserving Encoders for GAE and VGAE}}
\label{s32}

Following this diagnosis and overview, we now provide more detail on the first of the three bullet points in Section \ref{s31}, i.e., our proposed revised \textit{encoding}~strategy. We recall that our proposed solution aims to encode nodes as embedding vectors $z_i$  \textit{more suitable for community detection}. Essentially, intrinsic communities in the graph under consideration should be easily retrievable from these representations, e.g., from their $L_2$ distances via a straightforward $k$-means clustering. These vectors should also simultaneously remain relevant for \textit{link prediction,} i.e., as for existing GAE and VGAE, the likelihood of a missing edge between two nodes should also be inferred from the learned representations $z_i$. In the following, for consistency with previous work (see Section~\ref{s2}), we continue using the inner product $\hat{A}_{ij} = \sigma(z^T_i z_j) \in [0,1]$ as the probability of an edge between nodes $i$ and $j$.

\subsubsection{Revising the Message Passing Operator}

Existing graph encoders usually involve normalized versions of the adjacency matrix~$A$, or some generalized \textit{message passing operator} matrix that also captures each node's direct connections in the graph under consideration \cite{Dasoulas2021}. For instance, in the popular multi-layer GCN in Equation~\eqref{eq:gcn}, the symmetric normalization $\agcn{A}$ from Equation~\eqref{eq:norm} is used such that at each layer $l$ a vectorial representation for each node is computed by taking a weighted average of the representations from layer
$l - 1$ of its direct neighbors and of itself. In this work, we adopt an alternative strategy that consists in computing the weighted average of, at each layer:
\begin{itemize}
    \item representations from the direct neighbors of each node, as above;
    \item but also representations from other \textit{unconnected nodes} that, according to some prior available knowledge and criteria, belong to the same graph community. 
\end{itemize}

More precisely, let us assume that we have, at our disposal, a preprocessing \textit{graph mining} technique that, based on the graph structure and on some fixed criteria, learns an initial \textit{prior partition of the node set} $\mathcal{V}$ into $\ncluster$ sets $C_1, \ldots, C_\ncluster,$ with $|C_\clustersubscript| = n_\clustersubscript$ for $\clustersubscript \in \{1, \ldots, \ncluster\}.$ Here, $\ncluster$ acts as a hyperparameter, that can differ from the actual number of communities eventually used for the community detection downstream evaluation task (i.e., the number of clusters in the $k$-means operated on the final vectors $z_i$). A concrete example of such a technique will be provided in Section~\ref{s323} (together with explanations on how to select $\ncluster$). We simply assume its availability throughout these paragraphs.

We propose to leverage such an initial partition as a \textit{prior node clustering signal} from which the GAE/VGAE encoder should benefit, but also have the ability to deviate during training, when learning the embedding space. Specifically, we propose to replace the standard input adjacency matrix $A$~by:
\begin{equation}
A + \lambda \Ac,   
\end{equation}
where $\lambda \geq 0$ is a scalar hyperparameter, and where $\Ac$ is the community membership matrix defined as follows:

\begin{definition}\label{def:Ac}
Let us consider a partition of the node set $\mathcal{V}$ into $\ncluster$ sets $C_1, \ldots, C_\ncluster$. The corresponding \textit{community membership matrix} is defined as:
\begin{equation}
A_c=MM^T-I_n,
\label{eqAc}
\end{equation}
 with $M\in\{0,1\}^{n\times \ncluster}$ denoting the $n\times \ncluster$ matrix where elements $M_{i\clustersubscript}=1$ if and only if $i\in C_\clustersubscript$ according to the prior clustering.
\end{definition}
We interpret $\Ac$ as the adjacency matrix of an alternative graph in which each cluster of our prior partition is represented by a fully connected graph, without self-loops. Since nodes are only allocated to one cluster, there exists a node ordering such that the matrix $\Ac$ is block-diagonal. In essence, $A + \lambda \Ac$ aims to capture refined node similarities, by simultaneously considering some \textit{local} information from direct neighborhoods, and some \textit{global} information from prior node communities. The hyperparameter $\lambda$ helps to balance these two aspects. In particular, setting $\lambda = 0$ results in the standard adjacency matrix.

\subsubsection{From Message Passing Operators to Encoding Schemes}
\label{s322}

At first glance, $A + \lambda \Ac$ could straightforwardly be incorporated as a refined message passing operator in popular GAE and VGAE encoders. For instance, one could consider its direct incorporation in:
\begin{itemize}
    \item variants of \textit{2-layer GCN encoders}, initially proposed by Kipf and Welling \cite{kipf2016-2}, as this neural architecture remains the most popular GAE/VGAE encoder in the literature \cite{choong2018learning,pan2018arga,salha2019-2,huang2019rwr,grover2019graphite,semiimplicit2019,aaai20,salha2021fastgae,pei2021generalization}. Specifically, one could consider:
    \begin{itemize}
        \item a version incorporating $A + \lambda \Ac$ in both layers. Then, for example the GAE formulation\footnote{For clarity of exposition we discuss the deterministic GAE framework (Section~\ref{sec:GAE}). However, the changes are equally applicable to the VGAE framework (Section~\ref{sec:VGAE}), for which $Z$ has to be replaced by $\mu$ and $\log \sigma$ as in Equation~\eqref{vaeencoder}.} in Equation~\eqref{eq:gcn} becomes, $Z = \text{GCN}^{(1)}(A+\weight\Ac, X) = \agcn{A+\weight\Ac} \text{ReLU} (\agcn{A+\weight\Ac} X W^{(0)}) W^{(1)}$.
        \item a version incorporating the prior communities only on the first layer, i.e.,  $Z = \text{GCN}^{(2)}(A+\weight\Ac, X) = \agcn{A} \text{ReLU} (\agcn{A+\weight\Ac} X W^{(0)}) W^{(1)}$.
    \end{itemize}
    \item or, a variant of the \textit{linear encoder}\footnote{A ``\textit{linear encoder}'' is actually a particular case of GCN with a single layer and without activation function. For consistency with previous works \cite{salha2020simple,waradpande2020graph,shin2020bipartite,salha2019keep}, we nonetheless adopt the ``\textit{linear encoder}'' naming in this work, and use ``\textit{GCN}'' to refer to the above \textit{multi-layer} graph convolutional networks.} proposed by Salha et al. \cite{salha2020simple}. Indeed, this simplified one-hop model without activation reached competitive performances w.r.t. multi-layer GCNs for GAE/VGAE-based community detection in recent studies \cite{choong2020optimizing,salha2020simple}. In this case: $Z = \text{Linear}(A+\weight\Ac, X) = \agcn{A+\weight\Ac} X W^{(0)}$.
\end{itemize}

However, the computational cost of evaluating each layer of a GCN or a linear encoder depends linearly on the number of edges $|\mathcal{E}| = m$ in the message passing operator \cite{salha2020simple,kipf2016-1}. As the graph represented by $A+\weight \Ac$ contains at least $\sum_{\clustersubscript=1}^\ncluster n_\clustersubscript^2$ edges, such a direct incorporation of $A+\weight \Ac$ in encoders could incur a large computational expense. 

To alleviate this cost, we will instead consider a \textit{$\dregc$-regular sparsification of $\Ac$}, denoted by $\Ao$ in the following. In $\Ao$, each node $i \in C_\clustersubscript$ is only connected to $\dregc < n_\clustersubscript$ randomly selected nodes in $C_\clustersubscript$ (instead of all other nodes in $C_\clustersubscript$). Therefore, the $A + \lambda \Ao$ message passing operator still contains some of the prior clustering information without necessarily incurring 
the cost implied by the use of $\Ac$. In particular, selecting $\dregc \approx \frac{m}{n}$ ensures that $A + \lambda \Ao$ has $O(2m)$ non-null elements, preserving the linear complexity w.r.t. $m$ of the aforementioned encoders (in our experiments, the optimal value of $s$ will be selected for each graph as described in Section~\ref{s413}). Note that we only sample $\Ao$ once at the beginning of the model training and then keep it fixed throughout training and testing. To sum up, in our upcoming experiments in Section~\ref{s4} we will instead consider the following two\footnote{We will favor  $\text{GCN}^{(2)}$ over $\text{GCN}^{(1)}$ in the remainder of this work, as the former outperformed the latter in our experiments. To simplify the notation $\text{GCN}^{(2)}$ will be referred to as GCN~in~experiments.} encoding schemes:
\begin{itemize}
    \item $Z = \text{GCN}^{(2)}(A+\weight \Ao, X) = \agcn{A} \text{ReLU} (\agcn{A+\weight \Ao} X W^{(0)}) W^{(1)}$.
    \item $Z = \text{Linear}(A+\weight \Ao, X) = \agcn{A+\weight \Ao} X W^{(0)}$.
\end{itemize}

In these encoders, our altered message passing scheme allows practitioners to incorporate information from prior communities in the resulting node embedding space. A given node $i \in \mathcal{C}_\clustersubscript \subset \mathcal{V}$, for $\clustersubscript \in \{1,...,\ncluster\},$ will aggregate information from its direct neighbors and from some nodes in  $\mathcal{C}_\clustersubscript$. By design, $i$ will thus have an embedding vector $z_i$  more similar to the embedding vectors of the other nodes in $\mathcal{C}_\clustersubscript$ than would be the case for the standard encoders based on $\agcn{A}$.
We recall that the choice of linear and 2-layer GCN encoders is made without loss of generality. $A + \lambda  \Ao$ could be incorporated into other encoders including deeper GCNs, ChebNets \cite{defferrard2016} or Graph Attention Networks \cite{velivckovic2019graph}.

The remainder of this Section~\ref{s32} on encoders is organized as follows. In Section~\ref{s323}, we now detail how we derive the matrix $\Ac$ (that has loosely been assumed to be ``available'' so far) in our work. Then, in Section~\ref{s324}, we provide a theoretical analysis of our novel encoding strategy. It notably aims to better understand our newly introduced operators $\Ac$ and $\Ao$ in terms of the spectral filtering they induce, as well as to assess the impact of the $\dregc$-regular sparsification of $\Ac.$

\subsubsection{Learning $\Ac$ and $\Ao$ with Modularity-Based Clustering}
\label{s323}

So far, for pedagogical purposes, we loosely assumed the availability of the $\Ac$ and $\Ao$ prior community membership matrices. In practice, how these matrices are \textit{learned} plays an important role, as the empirical performance of our strategy will directly depend on the quality of the underlying prior node clusters. Throughout this paper, we will rely on \textit{modularity} concepts to learn $\Ac$ -- hence the name \textit{Modularity-Aware GAE and VGAE}. More specifically, we will leverage the popular \textit{Louvain} algorithm \cite{blondel2008louvain}.

In the absence of node feature information, the Louvain greedy algorithm remains a popular and powerful approach for community detection \cite{blondel2008louvain}.
It iteratively aims to maximize the \textit{modularity} value \cite{Newman8577}, defined as follows:

\begin{definition}\label{def:modularity}
Let us consider a graph $\mathcal{G} = (\mathcal{V}, \mathcal{E})$ with adjacency matrix $A$ and nodes $i\in \mathcal{V}$ of degree $d_i = \sum_{j=1}^n A_{ij}$. We denote a partition of these nodes into $\ncluster \leq |\mathcal{V}|$ communities by $\{C_1,...,C_\ncluster\}.$ Then, the \textit{modularity} associated to this partition is:
\begin{equation}\label{eq:modularity}
Q = \frac{1}{2m} \sum_{i,j=1}^n \left[A_{ij} - \frac{d_id_j}{2m}\right] \delta(i,j),
\end{equation}
where $m$ is the sum of all edge weights in the graph (i.e., the number of edges for unweighted graphs), and where $\delta(i,j) = 1$ if nodes $i$ and $j$ belong to the same community and $0$ otherwise. 
\end{definition}
In essence, the modularity compares the density of connections inside communities to connections between communities. More specifically, Equation \eqnref{eq:modularity} returns a scalar $Q$ in the range $[-\frac{1}{2}, 1]$, that measures the difference between the observed fraction of (potentially weighted) edges that occur within the same community and the expected fraction of edges in a configuration model graph, which matches our observed degree distribution but allocates edges randomly without any specified community structure.

In the Louvain greedy algorithm, aiming to maximize the modularity of a graph over the set of possible cluster assignments, each node is initialized in its own community. Then, the algorithm iteratively completes two phases:

\begin{itemize}
    \item In phase 1, for each node $i \in \mathcal{V}$, one computes the change in modularity resulting from the allocation of node $i$ to the community of each of its neighbors. Then, $i$ is either placed into the community leading to the greatest modularity increase, or remains in its original group if no increase is possible;
\item In phase 2, a new graph is constructed. Nodes correspond to communities obtained in phase 1, and edges are formed by summing edge weights occurring between communities. Edges within a community are represented by self-loops in this new graph. One repeats phase 1 on this new graph until no further modularity improvement is possible.
\end{itemize}

\paragraph{\textbf{Why using the Louvain method?} }Our justification for the use of this method to derive $\Ac$ is threefold.
\begin{itemize}
    \item First and foremost, it automatically selects the relevant number of prior communities $\ncluster$, by iteratively maximizing the modularity value.
    \item Secondly, it runs in $O(n\log n)$ time  \cite{blondel2008louvain}, with $n = |\mathcal{V}|$. Therefore, it scales to large graphs with millions of nodes, such as those in our experiments in Section~\ref{s4}.
    \item Thirdly, such a modularity criterion complements the encoding-decoding paradigm of standard GAE and VGAE. We argue that learning node embedding spaces from complementary criteria is beneficial. Our experiments will confirm that leveraging prior modularity-based node clusters in the GAE/VGAE outperforms the individual use of the Louvain or of the GAE/VGAE \textit{alone}.
\end{itemize}  Note that, the use of the Louvain method is made without loss of generality as our framework remains valid for alternative graph mining methods deriving $\Ac$ and $\Ao$.

\subsubsection{Theoretical Analysis of the Encoder's Message Passing Operator}
\label{s324}


We now conduct a theoretical analysis of our newly introduced message passing operator, which we begin by motivating the spectral analysis of the matrices involved.
Recall, from Equation \eqnref{eq:gcn}, that the computations performed by a GCN at a given layer are the following,
\begin{equation} \label{eq:GCN_layer}
    \text{ReLU} (\agcn{A} H^{(l-1)} W^{(l-1)}).
\end{equation}
If we consider the spectral decomposition of the message passing operator that is used in Equation \eqnref{eq:GCN_layer}, $\agcn{A} = U\Theta U^T,$ where $U=[u_1, \ldots, u_n]$ denotes the matrix containing the eigenvectors $u_i$ of $\agcn{A}$ and $\Theta$ is a diagonal matrix containing the eigenvalues $\theta_i$ of $\agcn{A}.$ Then, the computation performed in Equation \eqnref{eq:GCN_layer} can be reformulated to be,
\begin{equation} \label{eq:GCN_spectral}
    \text{ReLU} (U\Theta U^T H^{(l-1)} W^{(l-1)}) = \text{ReLU} \left(\sum_{i=1}^n \theta_i u_iu_i^T H^{(l-1)} W^{(l-1)}\right).
\end{equation}
Therefore, performing one message passing step of the hidden states $H$ on a graph given by $\agcn{A},$ i.e., $\agcn{A} H^{(l-1)},$ can be interpreted as a Fourier transform of $H,$ called \textit{graph Fourier transform} \citep{Shuman2013}, where the eigenvectors of $\agcn{A}$ act as a Fourier basis and the eigenvalues of $\agcn{A}$ define the Fourier coefficients. 

When trying to perform a theoretical analysis of the message passing step in Equation~\eqnref{eq:GCN_layer} it often turns out to be more insightful to consider Equation \eqnref{eq:GCN_spectral} instead and analyze the eigenvalues and eigenvectors of the used message passing operator. Such a spectral perspective has 
given rise to a variety of architectures proposing learnable functions applied to the diagonal terms of $\Theta$ \citep{Bruna2014, defferrard2016, Levie2019}. 
Historically, the study of spectral graph theory \citep{Chung1997, Spielman2012}, and in particular the area of graph signal processing \citep{Sandryhaila2014, Ortega2018}, has yielded much insight in the study of graphs and therefore it is somewhat unsurprising that also in the study of the GNNs the spectral analysis of these architectures is a promising avenue of analysis \citep{gama2020,Balcilar2021, Dasoulas2021}. 

We, therefore, now provide spectral results allowing us to gain a better understanding of our proposed message passing operator and compare our proposed message passing operator to the standard message passing operators.  To characterize the eigenvectors of our newly introduced $\agcn{\Ac}$ we rely on the concept of 2-sparse eigenvectors. 

\begin{definition}
\citep{Teke2017}
The entries of \textit{2-sparse eigenvectors} are all equal to $0$ except for the $i^{\mathrm{th}}$ and $j^{\mathrm{th}}$ entry which equal to $1$ and $-1,$ where $i$ and  $j$ denote two nodes which share all their neighbors, i.e., $A_{ih}=A_{jh}$ for $h\in\{1,\ldots,n\}\backslash\{i,j\}.$ 

\end{definition}
An extended discussion of the literature related to such 2-sparse eigenvectors and their corresponding vertices, which are sometimes referred to as twin vertices, can be found in \citet[p.48-9]{Lutzeyer2020}. 
We are now able to characterize the spectrum and eigenvectors of $\agcn{\Ac}.$

\begin{proposition} \label{thm:Acspectrum}
The matrix $\agcn{\Ac}$ has eigenvalues $\{\{1\}^{\ncluster}, \{0\}^{n-\ncluster}\},$ where we denote the multiset containing a given element $x,$ $y$ times, by $\{x\}^y.$  Each non-zero eigenvalue has an associated eigenvector $v_\clustersubscript,$ with $\clustersubscript \in \{1, \ldots, \ncluster\},$  with entries $(v_\clustersubscript)_i=1$ for $i \in C_\clustersubscript$ and $(v_\clustersubscript)_i=0$ for $i \not\in C_\clustersubscript.$ 
The eigenspace corresponding to the zero eigenvalue has dimension $n-\ncluster$ and is spanned by, for example, a set of two-sparse eigenvectors on each of the connected components in the graph.

\end{proposition}


The proof of Proposition \ref{thm:Acspectrum} can be found in \ref{app:proof_Acspectrum}. 
The informal take-away from Proposition \ref{thm:Acspectrum} is that \textit{the cluster membership of nodes is encoded clearly and compactly in the spectrum and eigenvectors of $\agcn{\Ac}.$ }
More formally, in Proposition \ref{thm:Acspectrum} we observe that in the spectral domain the operator $\agcn{A_c},$ which we introduce to the encoder's message passing scheme, directly encodes the cluster membership of the different nodes and all other signals are filtered out by the $0$ eigenvalues. Also in the graph domain the matrix $\agcn{\Ac}$ clearly encodes the cluster structure by representing each cluster by a fully connected component of the graph. Therefore, the matrix $\agcn{\Ac}$ is an appropriate choice to introduce cluster information into the message passing scheme and does so clearly in both the graph and spectral domains.

In general, the spectral filtering performed by our message passing operator, $\agcn{A +\weight \Ac},$ and the standard message passing operator, $\agcn{A},$ are different. $\agcn{A +\weight \Ac}$ accounts for the clustering information which we introduce. In the following theorem we provide a result that allows us to establish under which conditions the spectral filtering performed by $\agcn{A}$ and $\agcn{A +\weight \Ac}$ are equal, to gain a better understanding of the action of $\agcn{A +\weight \Ac}.$

\begin{proposition}\label{thm:spectral_relation_A_Ac_AAc}
If $\mathcal{G}$ is composed of regular connected components, i.e., connected components containing only vertices of equal degree, and the partition of the node set defining these regular components equals the partition defining $\Ac,$ then the matrices $\agcn{\Ac}$ and $\agcn{A+\weight \Ac}$ have a shared set of eigenvectors and the spectrum of $\agcn{A+\weight \Ac},$ denoted by $\mathcal{S}(\agcn{A+\weight \Ac}),$ can be expressed in terms of the eigenvalues of $\agcn{A}$ and $\agcn{\Ac},$ denoted by $\theta$ and $\eta$, respectively,  as follows, 
$$
\mathcal{S}(\agcn{A+\weight \Ac}) = \{g_1(\theta_1) + g_2(\eta_{s(1)}), \ldots, g_1(\theta_n) + g_2(\eta_{s(n)})\},
$$
for affine functions $g_1, g_2$ parametrised by the node degrees and some permutation $s(\cdot)$ defined on the set $\{1, \ldots, n\}.$
\end{proposition}


The proof of Proposition \ref{thm:spectral_relation_A_Ac_AAc} can be found in \ref{app:proof_spectral_relation_A_Ac_AAc}. 
Hence, we observe that for graphs consisting of regular connected components the spectral filtering performed by our proposed message passing operator $\agcn{A+\weight \Ac}$ is equal to that of the standard operator $\agcn{A}.$ For graphs consisting of regular connected components the clustering information is already contained in the spectrum of $\agcn{A}$ and therefore its further addition does not affect the eigenvectors of our proposed message passing operator. 

Note that, in general, the spectrum of the sum of two matrices cannot be characterized by the individual spectra of the two matrices, meaning that, in general, there does not exist an exact relation between the spectra of $\agcn{A}$ and $\agcn{\Ac}$ to $\agcn{A+\lambda\Ac}.$ We can however, make direct use of existing results such as Weyl's inequality \cite{Weyl1912,Kolotilina2005} and the extended Davis--Kahan theorem \cite{Lutzeyer2019}, which, respectively, upper bound the distance of the eigenvalues and spaces spanned by the eigenvectors of the sum of matrices and the individual matrices. 


\paragraph{\textbf{$\dregc$-regular sparsification of our message passing operator}}

We now turn to the analysis of our sparsified message passing operator  
$\Ao,$ which, as we will see now, still contains the external cluster information without incurring the large computation cost implied by the use of $\Ac.$ 

\begin{proposition} \label{thm:Aospectrum}
If the partition defining the connected components of $\Ao$ is a refinement of the partition defining the components of $\Ac$,
then the multiplicity of the largest eigenvalue of $\agcn{\Ao}$ is greater or equal to the multiplicity of the largest eigenvalue of $\agcn{\Ac}.$ Further, the largest eigenvalue of both $\agcn{\Ao}$ and $\agcn{\Ac}$ equals 1 and the eigenvectors corresponding to the eigenvalue 1 of $\agcn{\Ac}$ are also eigenvectors corresponding to the eigenvalue 1 of $\agcn{\Ao}.$

\end{proposition}


The proof of Proposition \ref{thm:Aospectrum} can be found in \ref{app:proof_Aospectrum}. 
Informally, Proposition \ref{thm:Aospectrum} can be interpreted to show that \textit{the sparsification of $\Ac$ producing $\Ao$ does not impact the ``informative'' part of the spectrum. } 
Recall, that the eigenvectors corresponding to the largest eigenvalue of $\agcn{\Ac}$ and $\agcn{\Ao}$ are indicator vectors of our introduced cluster membership. Since the remaining eigenvectors are orthogonal to these indicator vectors we know that none of them encode our cluster membership as compactly as the eigenvectors corresponding to the largest eigenvalue. 

For $\agcn{\Ao}$ the eigenvalues corresponding to the less informative eigenvectors correspond to nonzero eigenvalues in general and we expect the choice of $\dregc$ to influence the impact of this uninformative part of the spectrum. This uninformative part of the spectrum can be upper bounded by adapting the bound in \citet{Friedman2003} to our message passing operator. However, in our work we choose a more practice-oriented approach by treating $\dregc$ as a hyperparameter of our model and find its optimal values using the procedure which is described in the upcoming Section 
\ref{s33}.

\subsection{\textbf{Decoding and Training Strategies}}
\label{s33}

So far, our work mainly considered improvements of the encoder's \textit{architecture}. While this aspect is crucial, we also argue that previously proposed models were often \textit{optimized} in a fashion that, by design, favors link prediction over community detection. With this in mind, this Section~\ref{s33} now complements our contributions from Section~\ref{s32} with revised training and optimization strategies.

\subsubsection{Modularity-Inspired Losses for GAE and VGAE}
\label{s331}

As explained in Section~\ref{s2}, neural weight matrices of previous GAE and VGAE encoders were tuned by optimizing \textit{reconstruction losses}, capturing the similarity between the decoded graph and the original one. Usually, these losses directly evaluate the quality of reconstructed node pairs $\hat{A}_{ij}$ w.r.t. their ground-truth counterpart $A_{ij}$. This includes the cross-entropy loss $\lossGAE$ from Equation~\eqref{lossGAE} and the ELBO loss $\lossVGAE$ from Equation~\eqref{elbo}. We argue that this optimization strategy also contributes to explaining the underwhelming performance of some GAE and VGAE models on community detection~tasks.
\begin{itemize}
    \item By design, existing optimization strategies favor good performances on link prediction tasks, that precisely consist in accurately reconstructing connected/unconnected node pairs. However, some recent studies emphasized that a good reconstruction of \textit{local} pairwise connections does not always imply a good reconstruction of the \textit{global} community structure from the graph under consideration \cite{wang2017community,liu2019much}. This motivates the need for a revised loss function capturing some global community information.
    \item Besides, GAE/VGAE-based community detection experiments often consisted in running $k$-means algorithms in the final node embedding space (and, as stated at the beginning of Section~\ref{s32}, we also adopt this strategy). However, this results in clustering embedding vectors based on their $L_2$ \textit{distances} $\Vert z_i - z_j \Vert_2 $, whereas the aforementioned reconstruction losses instead often involve \textit{inner products} ($\hat{A}_{ij} = \sigma(z^T_i z_j)$). There is thus a discrepancy between the criterion ultimately used for $k$-means clustering, and the one used during training to assess node similarities.
\end{itemize}
To address these issues, we propose to complement standard GAE and VGAE losses with an additional loss term, involving $L_2$ distances and inspired by the \textit{modularity}\footnote{We emphasize that this new term does \textit{not} involve the prior Louvain clusters used in $\Ac$.} in Equation~\eqref{eq:modularity}. In the case of the GAE, we will iteratively minimize by gradient descent:
\begin{equation}
 \lossMAGAE = \lossGAE - \frac{\beta}{2m} \sum_{i,j=1}^n \left[A_{ij} - \frac{d_id_j}{2m}\right]  e^{-\gamma \Vert z_i - z_j \Vert^2_2},
 \label{eq:tildeAE}
 \end{equation}
with hyperparameters $\beta \geq 0$, $\gamma \geq 0.$ Also, for VGAE we will maximize\footnote{We recall that $\lossGAE$ is \textit{minimized} while $\lossVGAE$ is \textit{maximized}, hence the occurrence of a \textit{minus} term in Equation \eqref{eq:tildeAE} but a \textit{plus} term in Equation \eqref{eq:tildeVAE}.}:
\begin{equation}
\lossMAVGAE = \lossVGAE + \frac{\beta}{2m} \sum_{i,j=1}^n \left[A_{ij} - \frac{d_id_j}{2m}\right] e^{-\gamma \Vert z_i - z_j \Vert^2_2}. \label{eq:tildeVAE}
\end{equation}

We note that $\lossMAGAE$ and $\lossMAVGAE$ are independent of the ground-truth community information. Consistently with previous efforts on community detection using GAE and VGAE (see Section~\ref{sec:comm_detection}), we only use ground-truth communities  for the final \textit{evaluation} of embedding vectors -- not to learn these vectors. 

In Equations \eqref{eq:tildeAE} and \eqnref{eq:tildeVAE}, the exponential term (taking values in $[0,1]$) acts as a \textit{soft counterpart of the common community indicator} $\delta(i,j) \in \{0,1\}$ in Equation~\eqref{eq:modularity}. It tends to 1 when nodes $i$ and $j$ get closer in the embedding space, and tends to 0 when they move apart. 

In essence, we expect the addition of such a \textit{global regularizer} to $\lossGAE$ and $\lossVGAE$ to encourage closer embedding vectors (in the $L_2$ distance) of densely connected parts of the original graph, and therefore to permit a  $k$-means-based \textit{detection of communities with higher modularity values}. On the other hand, the remaining presence of the original $\lossGAE$ or $\lossVGAE$ term\footnote{Our experiments will consider the original $\lossGAE$ or $\lossVGAE$ from Equations \eqref{lossGAE} and \eqref{elbo} and originally formulated by Kipf and Welling \cite{kipf2016-2}. Nonetheless, one can observe that our modularity-inspired global regularizer term could be optimized \textit{in conjunction with other reconstruction losses.} For instance, modularity-inspired terms could be added to the variant formulation of ELBO loss from Choong et al. \cite{choong2018learning}, that incorporates \textit{Gaussian mixtures} in the Kullback-Leibler divergence.} in the loss aims to \textit{preserve good performances on link prediction}. The hyperparameter $\beta$ balances the relative importance of the modularity regularizer w.r.t. the pairwise node pairs reconstruction loss, while the hyperparameter $\gamma$ regulates the magnitude of $\Vert z_i - z_j \Vert^2_2$ in the exponential term. Our experiments will show that proper tuning of $\beta$ and $\gamma$ permits us to improve community detection while jointly preserving performances on link prediction.

The use of a modularity-inspired regularizer in the loss of the \textit{Modularity-Aware GAE and VGAE} builds upon several studies, which were not studying the GAE/VGAE frameworks but emphasized the benefits of various modularity-inspired losses for learning community-preserving node embedding representations \cite{wang2017community,yang2016modularity,lobov2019unsupervised}. In our setting, we favor the use of a \textit{soft} modularity instead of the term in Equation~\eqref{eq:modularity}, as it permits 1) to obtain a differentiable loss, and 2) to avoid the actual reconstruction of node communities at each training iteration, which would incur a larger computational expense.

To conclude we note that, as for a complete evaluation of $\lossGAE$ and $\lossVGAE$, computing the modularity-inspired terms in Equations~\eqref{eq:tildeAE}~and~\eqref{eq:tildeVAE} \textit{on the entire graph} would be of quadratic complexity w.r.t. the number of nodes in the graph. In some of our experiments where such a complexity would be unaffordable (roughly, when $n \geq$ 50 000), we will rely on the FastGAE method \cite{salha2021fastgae} to approximate modularity-inspired terms. This method was already mentioned in Section~\ref{s213}, to approximate $\lossGAE$ and $\lossVGAE$ losses on large graphs by computing them only on strategically-selected random sampled \textit{subgraphs} of $O(n)$ size, drawn at each training iteration. We will subsequently approximate modularity-inspired terms on the same subgraphs. This ensures a linear complexity for each evaluation of our proposed $\lossMAGAE$ and $\lossMAVGAE$ losses in Equations~\eqref{eq:tildeAE}~and~\eqref{eq:tildeVAE}, respectively. As our revised encoders from Section~\ref{s32} also exhibit a linear complexity w.r.t. $n,$ our whole \textit{Modularity-Aware GAE and VGAE} are scalable to graphs with hundreds of thousands to millions of nodes.

%
%
%




\subsubsection{On the Selection of Hyperparameters}
\label{s332}

We expect our modularity-inspired losses to improve the training of our linear/GCN encoders for community detection, i.e., the tuning of their \textit{weight matrices}. However, in addition to these weight matrices, our \textit{Modularity-Aware GAE and VGAE} involve several other hyperparameters, that also play a key role. This includes the standard hyperparameters of GAE and VGAE models (e.g., the number of training iterations, the learning rate, the dimensions of encoding layers and, potentially, the dropout rate \cite{srivastava2014dropout}), but also our newly introduced hyperparameters:  $\lambda$ and $\dregc$ from our encoders, as well as $\beta$ and $\gamma$ from our losses.

In previous research, the selection procedure for such important hyperparameters was sometimes solely based on the optimization of AUC or AP scores on link prediction \textit{validation} sets \cite{salha2019-1,salha2021fastgae}, following the train/validation/test splitting procedure initially adopted by Kipf and Welling \cite{kipf2016-2} and previously described in Section~\ref{s231}. However, intuitively, the best hyperparameters for community detection might differ from the best ones for link prediction. Such a selection procedure might therefore be suboptimal for community detection problems.

To tackle this issue, and to complement our novel encoders (Section~\ref{s32}) and losses (Section~\ref{s331}), we propose an alternative hyperparameter selection procedure w.r.t. previous practices. As community detection is an \textit{unsupervised} downstream task, we cannot rely on train/validation/test splits as for the \textit{supervised} link prediction binary classification task\footnote{We recall that the ground-truth communities of each node will be \textit{unavailable} during training. They will only be ultimately revealed for model evaluation, to compare the agreement of the node partition proposed by our GAE or VGAE model to the ground-truth partition.}. Consistent with our already described contributions, we rather propose to rely on \textit{modularity} scores, as it is an unsupervised criterion computed independently of the unobserved ground-truth clusters.
More precisely, to select relevant hyperparameters, we will:
\begin{itemize}
        \item firstly, construct link prediction train/validation/test sets, as in Section~\ref{s231};
        \item then, select hyperparameters that maximize the average of:
        \begin{itemize}
            \item the \textit{AUC score} obtained for link prediction on the validation set; 
            \item the \textit{modularity score} $Q$ defined in Equation~\eqref{eq:modularity}. This score is obtained from the communities extracted by running a $k$-means on the final vectors $z_i,$ learned from the train graph (all nodes are visible but some edges - the validation and test ones - are masked).
        \end{itemize}
    \end{itemize} 
We expect this dual criterion 
to facilitate the identification of hyperparameters that will be jointly relevant for link prediction and community detection downstream applications.

\section{Experimental Evaluation}
\label{s4}

We now present an in-depth experimental evaluation of our proposed \textit{Modularity-Aware GAE and VGAE} models together with relevant baselines. In Section~\ref{s41} we first describe our experimental setting. Then in Section~\ref{s42}, we report and discuss our results. 

\subsection{\textbf{Experimental Setting}}
\label{s41}

\subsubsection{Datasets}
\label{s411}

In the following, we provide an experimental evaluation on seven graphs of various origins, characteristics, and sizes. Their statistics are summarized in Table~\ref{tab:datasetstats}. 

\begin{table}[t]
\centering
\caption{Statistics of graph datasets}
\vspace{0.2cm}
\label{tab:datasetstats}
   \resizebox{0.85\textwidth}{!}{
\begin{tabular}{c|c|c|c}
\toprule
\textbf{Dataset} & \textbf{Number of nodes} & \textbf{Number of edges} & \textbf{Number of communities} \\
\midrule
\midrule
\textbf{Blogs} & 1 224  & 19 025 & 2  \\
\textbf{Cora} & 2 708 & 5 429 & 6 \\
\textbf{Citeseer} & 3 327 & 4 732 & 7 \\
\textbf{Pubmed} & 19 717 & 44 338 & 3 \\
\textbf{Cora-Large} & 23 166  & 91 500 & 70 \\
\textbf{SBM} & 100 000 & 1 498 844 & 100 \\
\textbf{Deezer-Album} &  2 503 985  & 25 039 155 & 20 \\
\bottomrule
\end{tabular}}

\end{table}

First and foremost, we consider the \textit{Cora, Citeseer and Pubmed citation networks} \cite{kipf2016-1}. We study two versions of each of these datasets, \textit{with} and \textit{without} node features that correspond to bag-of-words vectors of dimensions $f =$ 1433, 3703 and 500, respectively. In these datasets, nodes are clustered in 6, 7
and 3 topic classes, respectively, that will act as the communities to be detected in our experiments. These three citations networks are by far the most commonly used graph datasets to evaluate GAE and VGAE models \cite{wang2017mgae,kipf2016-2,choong2018learning,tran2018multi,salha2019-1,huang2019rwr,grover2019graphite,semiimplicit2019,aaai20,choong2020optimizing,pei2021generalization,li2020dirichlet}. We therefore see value in studying them as well, especially in their \textit{featureless} version where, as explained in Section~\ref{sec:limits}, previous GAE and VGAE extensions fall short on community detection.

We complete our experimental evaluation with four other datasets. Firstly, we consider the ten times larger version of Cora used in \cite{salha2020simple} for community detection, and referred to as \textit{Cora-Large} in the following. Nodes are documents clustered in 70 topic-related communities. Additionally, we consider the \textit{Blogs web graph} also used in \cite{salha2020simple}, where nodes correspond to webpages of political blogs connected through hyperlinks. The blogs are clustered in two communities corresponding to politically left-leaning or right-leaning blogs.
Thirdly, as in Salha et al. \cite{salha2021fastgae}, we examine a graph generated from a \textit{stochastic block model}, which is a generative model for community-based random graphs \cite{abbe2017community}. We follow the parameterization of Salha et al. \cite{salha2021fastgae} and generate this graph, which we refer to as \textit{SBM}, as follows. Nodes are clustered in 100 ground-truth communities of 1000 nodes each. Nodes from the same community are connected with probability $p = 2 \times 10^{-2}$, while nodes from different
communities are connected with probability $q = 2 \times 10^{-4} < p$. Albeit being synthetic, this graph includes actual node communities by design, and is, therefore, relevant to evaluate community~detection~methods.

Lastly, we consider an industrial-scale private graph provided by the global music streaming service Deezer\footnote{\href{https://www.deezer.com/}{https://www.deezer.com/} (accessed October 13, 2021).}. Graph-based methods are at the core of Deezer's recommender systems \cite{salha2021cold}. In the graph under consideration in this study, nodes correspond to 2.5~million \textit{music albums} available on the service. They are connected through an undirected edge when they are regularly \textit{co-listened} by Deezer users (as assessed by internal usage metrics computed from millions of users,  but undisclosed in this work for privacy reasons). Deezer is jointly interested in 1) predicting new connections in the graph corresponding to new albums pairs that users would enjoy listening to together, which is achieved by performing the \textit{link prediction} task; and 2) learning groups of similar albums, with the aim of providing usage-based recommendations (i.e., if users listen to several albums from a community, other unlistened albums from this same community could be recommended to them), which is achieved by performing the \textit{community detection} task. In such an industrial application, learning high-quality album representations that would jointly enable effective link prediction and community detection would therefore be desirable. For evaluation, node communities will be compared to a ground-truth clustering of albums in 20 groups defined by their main \textit{music genre}, allowing us to assess the musical homogeneity of the node~communities~proposed~by~each~model.

\subsubsection{Tasks}
\label{s412}

For each of these seven graphs, we assess the performance of our models on two downstream tasks.
\begin{itemize}
    \item \textbf{Task 1:} We first consider a pure \textit{community detection} task, consisting in the extraction of a partition of the node set $\mathcal{V}$ which ideally agrees with the ground-truth communities of each graph. Communities will be retrieved by running the $k$-means algorithm (with $k$-means++ initialization \cite{arthur2017kmeans}) in the final embedding space of each model to cluster the vectors $z_i$ (with $k$ matching the known number of communities from Table~\ref{tab:datasetstats}); except for some baseline methods that explicitly incorporate another strategy to partition nodes (see Section \ref{s413}). We compare the obtained partitions to the ground-truth using the popular \textit{Adjusted Mutual Information
(AMI)} and \textit{Adjusted Rand Index (ARI)} scores\footnote{\label{footnote:sklearn} Scores are computed via scikit-learn \cite{pedregosa2011scikit}, using formulas provided here: \href{https://scikit-learn.org/stable/modules/classes.html?highlight=metrics\#module-sklearn.metrics}{https://scikit-learn.org/stable/modules/classes.html?highlight=metrics\#module-sklearn.metrics} (accessed October 13, 2021).} for clustering~evaluation. 

\item \textbf{Task 2:} We also consider a \textit{joint link prediction and community detection} task. In such a setting, we learn all node embedding spaces from \textit{incomplete} versions of the seven graphs, where 15\% of edges were randomly masked. We create a validation and
a test set from these masked edges (from 5\% and 10\% of edges, respectively, as in Kipf and Welling~\cite{kipf2016-2}) and the same number of randomly picked unconnected node pairs acting as ``non-edge'' negative pairs. Then, using decoder predictions $\hat{A}_{
ij}$ computed from vectors $z_i$ and $z_j,$ we evaluate each model's ability to distinguish edges from non-edges, i.e., \textit{link prediction}, from the embedding space, using the \textit{Area
Under the ROC Curve (AUC)} and \textit{Average Precision (AP)} scores\footref{footnote:sklearn} on the test sets. Jointly, we also evaluate the community detection performance obtained from such incomplete graphs, using the same methodology and AMI/ARI scores as in Task~1.
\end{itemize}

In the case of Task~2, we expect AMI and ARI scores to slightly decrease w.r.t. Task~1, as models will only observe \textit{incomplete} versions of the graphs when learning embedding spaces. Task~2 will further assess whether empirically improving community detection inevitably leads to deteriorating the original good performances of GAE and VGAE models on link prediction. As our proposed Modularity-Inspired GAE and VGAE are designed for \textit{joint link prediction and community detection}, we expect them to 1) reach comparable (or, ideally, identical) AUC/AP link prediction scores w.r.t. standard GAE and VGAE, while 2) reaching better community~detection~scores.
 
\subsubsection{Details on Models}
\label{s413}

For the aforementioned evaluation tasks and graphs, we will compare the performances of our proposed \textit{Modularity-Aware GAE and VGAE} models to standard GAE and VGAE and to several other baselines. All results reported below will verify $d = 16,$ i.e., all node embedding models will learn embedding vectors $z_i$ of dimension 16. We also tested models with $d \in \{32, 64\}$ by including them in our grid search space and reached similar conclusions to the $d = 16$ setting (we report and further discuss the impact of $d$ in Section~\ref{s42}. Note, the dimension $d$ is a selectable parameter in our public implementation, permitting direct model training on any node embedding dimension).

\paragraph{\textbf{Modularity-Aware GAE and VGAE}} We trained two versions of our Modularity-Aware GAE and VGAE: one with the \textit{linear encoder} described in Section~\ref{s322}, and one with the \textit{2-layer GCN encoder} ($\text{GCN}^{(2)}$). The latter encoder includes a 32-dimensional hidden layer. We recall that link prediction is performed from inner product decoding $\hat{A}_{ij} = \sigma(z^T_i z_j)$, and that community detection is performed via a $k$-means on the final vectors $z_i$ learned by each model. 

During training, we used the Adam optimizer~\cite{kingma2014adam}, without dropout (but we tested models with dropout values in $\{0,0.1,0.2\}$ in our grid search optimization). All hyperparameters were carefully tuned following the procedure described in Section~\ref{s332}. For each graph, we tested learning rates from the grid $\{0.001,0.005,0.01,0.05,0.1,0.2\}$, number of training iterations in $\{100, 200, 300, ..., 800\}$, with $\lambda \in \{0, 0.01,0.05, 0.1, 0.2, 0.3, ..., 1.0\}$,  $\beta \in \{0, 0.01,0.05, 0.1, 0.25, 0.5, 1.0, 1.5, 2.0\}$, $\gamma \in \{0.1, 0.2, 0.5, 1.0, 2, 5, 10\}$ and $\dregc \in \{1, 2, 5, 10\}$. The best hyperparameters for each graph are reported in Table~\ref{tab:hyperparameterstable}. We adopted the same optimal hyperparameters for GAE \textit{and} VGAE variants (a result which is consistent with the literature \cite{kipf2016-2}). Lastly, as exact loss computation was computationally unaffordable for our two largest graphs, SBM and Deezer-Album, their corresponding models were trained by using the FastGAE method \cite{salha2021fastgae}, approximating losses by reconstructing degree-based sampled subgraphs of $n =$ 10 000 nodes (a different one at each training iteration).

\begin{table}[t]
\centering
\caption{Complete list of optimal hyperparameters of Modularity-Aware GAE and VGAE models}
 \vspace{0.2cm}
    \label{tab:hyperparameterstable}
   \resizebox{\textwidth}{!}{
\begin{tabular}{c|cccccccc}
\toprule
\textbf{Dataset} & \textbf{Learning} & \textbf{Number of} & \textbf{Dropout} & \textbf{Use of FastGAE \cite{salha2021fastgae}} & $\lambda$ & $\beta$ & $\gamma$ & $\dregc$ \\
& \textbf{rate} & \textbf{iterations} & \textbf{rate} & \textbf{(if yes: subgraphs size)} & & & \\
\midrule
\midrule
\textbf{Blogs} & 0.01 & 200 & 0.0 & No & 0.5 & 0.75 & 2 & 10\\
\textbf{Cora (featureless)}  & 0.01 & 500 & 0.0 & No & 0.25 & 1.0 & 0.25 & 1\\
\textbf{Cora (with features)}  & 0.01 & 300 & 0.0 & No & 0.001 & 0.01 & 1 & 1\\
\textbf{Citeseer (featureless)}  & 0.01 & 500 & 0.0 & No & 0.75 & 0.5 & 0.5 & 2\\
\textbf{Citeseer (with features)}  & 0.01 & 500 & 0.0 & No & 0.75 & 0.5 & 0.5 & 2\\
\textbf{Pubmed (featureless)}  & 0.01 & 500 & 0.0 & No & 0.1 & 0.5 & 0.1 & 5\\ 
\textbf{Pubmed (with features)}  & 0.01 & 700 & 0.0 & No & 0.1 & 0.5 & 10 & 2\\
\textbf{Cora-Large}  & 0.01 & 500 & 0.0 & No & 0.001 & 0.1 & 0.1 & 10 \\
\textbf{SBM}  & 0.01 & 300 & 0.0 & Yes (10 000) & 0.5 & 0.1 & 2 & 10 \\
\textbf{Deezer-Album} & 0.005 & 600 & 0.0 & Yes (10 000) & 0.25 & 0.25 & 1 & 5\\
\bottomrule
\end{tabular}}
\end{table}

We used Tensorflow~\cite{abadi2016tensorflow}, training our models (as well as GAE/VGAE baselines described below) on an NVIDIA GTX 1080 GPU, and running other operations on a double Intel Xeon Gold 6134 CPU\footnote{On our machines, running times of the Modularity-Aware GAE and VGAE models were comparable to running times of their standard GAE and VGAE counterparts. For example, training each variant of VGAE on the Pubmed graph for 500 training iterations and with $\dregc = 5$ approximately takes 25 minutes on a single GPU (without the FastGAE method, which significantly speeds up training \cite{salha2021fastgae}). This is consistent with our claims on the comparable complexity of Modularity-Aware and standard models.}. Along with this paper, we will publicly release our source code on GitHub for reproducibility and to encourage future usage of our method\footnote{
\href{https://github.com/GuillaumeSalhaGalvan/modularity_aware_gae}{https://github.com/GuillaumeSalhaGalvan/modularity\_aware\_gae}}.

\paragraph{\textbf{Standard GAE and VGAE}} We compare the above Modularity-Aware GAE and VGAE to two variants of the standard GAE and VGAE: one with 2-layer GCN encoders with 32-dimensional hidden layer (which is equal to the seminal GAE and VGAE from Kipf and Welling~\cite{kipf2016-2}) and one with a linear encoder (which equals the linear GAE and VGAE from Salha et al.~\cite{salha2020simple}). We note that these are particular cases of our Modularity-Aware GAE/VGAE with GCN or linear encoder and with $\lambda = 0$ and $\beta = 0$. 

As for our Modularity-Aware models, link prediction is performed from inner product decoding, and community detection via a $k$-means on vectors $z_i.$ We also adopt a similar model selection procedure as for our Modularity-Aware GAE and VGAE to select hyperparameters (see Section~\ref{s332}). We selected similar learning rates and number of iterations to the values reported~in~Table~\ref{tab:hyperparameterstable}.

\paragraph{\textbf{Other baselines}}
For completeness, we also compare the standard and Modularity-Aware GAE/VGAE to several other relevant baselines. First and foremost, we report experiments on the VGAECD \cite{choong2018learning} and VGAECD-OPT \cite{choong2020optimizing} models, designed for community detection and discussed in Section~\ref{sec:CD_GAE_extensions}. We use our own Tensorflow reimplementation of these models\footnote{Authors of VGAECD/VGAECD-OPT did not release any public implementation of their models, and we were unable to reach them by e-mail. We note that we obtained some inconsistent results w.r.t. their original performances (specifically, we reached better performances on featureless graphs, and lower performances on graphs with node features), even when adopting their set of hyperparameters. For the sake of transparency, we will thereafter 1)~report scores obtained through our own re-implementation, and 2)~also specify scores reported in their original work, when they were significantly different. Authors followed an experimental setting and pure community detection task similar~to~ours.}. We set similar hyperparameters to the above other GAE/VGAE-based models. In all models, the number of Gaussian mixtures matches the ground-truth number of communities of each graph. 
Besides, we also report experiments on the DGVAE \cite{li2020dirichlet} model also discussed in Section \ref{sec:CD_GAE_extensions}, setting similar learning rates and layer dimensions to the above GAE/VGAE-based models, and using the authors' public implementation. In the case of DGVAE, we use 2-layer GCN encoders for consistency with other models of our experiments; we nonetheless acknowledge that \citet{li2020dirichlet} also proposed another encoding scheme, denoted Heatts in their paper (but unavailable in their public code at the time of writing) that could replace GCNs both in DGVAE and in Modularity-Aware GAE and VGAE. We also report experiments on the ARGA and ARVGA models from Pan et al. \cite{pan2018arga} that incorporate an adversarial regularization scheme, with similar hyperparameters and using the authors' implementation. ARGA and ARVGA emerged as some of the most cited GAE/VGAE extensions and, while they were not specifically introduced for community detection, Pan et al. \cite{pan2018arga} reported empirical gains on this task w.r.t.~standard~GAE/VGAE (on graphs with node features).

We furthermore consider three additional baselines not utilizing the autoencoder paradigm. Firstly, we report results obtained from the popular node embedding methods \textit{node2vec}~\cite{grover2016node2vec} and \textit{DeepWalk} \cite{perozzi2014deepwalk}. We used the authors' respective implementations, training models from 10 random walks with length 80 per node, window size of 5 and on a single epoch. For node2vec, we further set $p=q=1$. We use a similar strategy to our aforementioned GAE/VGAE models ($k$-means/inner products) for community detection and link prediction from embedding spaces. Lastly, we also compare to the \textit{Louvain} community detection method, using the authors' implementation~\cite{blondel2008louvain}. 
We see value in comparing our methods to a direct use of Louvain, as this method 1) often emerged as a simple but competitive alternative to GAE/VGAE for community detection (see Section~\ref{s242}), and 2) is directly leveraged in our proposed \textit{Modularity-Aware GAE/VGAE} as a pre-processing step for the computation of $A_c$ and $\Ao$ (see Section~\ref{s32}).

\subsection{\textbf{Results}}
\label{s42}

We now present our experimental results. In Section~\ref{s421} we analyze the impact of our proposed hyperparameter selection procedure. In Section~\ref{s422} and ~\ref{s423}, we discuss results on Task~1 and Task~2, respectively. Finally, we mention limitations and possible extensions of our approach in Section~\ref{s424}.

\subsubsection{On the Selection of Hyperparameters}
\label{s421}

In Section~\ref{s332}, we proposed an alternative hyperparameter selection procedure w.r.t. previous practices in the literature. Based on the joint maximization of AUC validation scores for link prediction and modularity scores $Q$, it aims to identify more relevant GAE/VGAE hyperparameters for joint link prediction and community detection. Recall, that the resulting optimal parameters are displayed in Table \ref{tab:hyperparameterstable}.

In our experiments, this procedure did not modify our choices of learning rates and dropout rates for the different GAE/VGAE models under consideration, w.r.t. a standard selection solely relying on AUC validation scores. It had a more noticeable impact on the choices of clustering-related hyperparameters in Modularity-Aware GAE and VGAE (i.e., $\lambda$, $\beta$, $\gamma$, and $\dregc$) as well as on the required number of training iterations in~the~gradient~descent. 

\begin{figure*}[t]
\centering
\resizebox{0.85\textwidth}{!}{
  \subfigure[Cora]{
  \scalebox{0.47}{\includegraphics{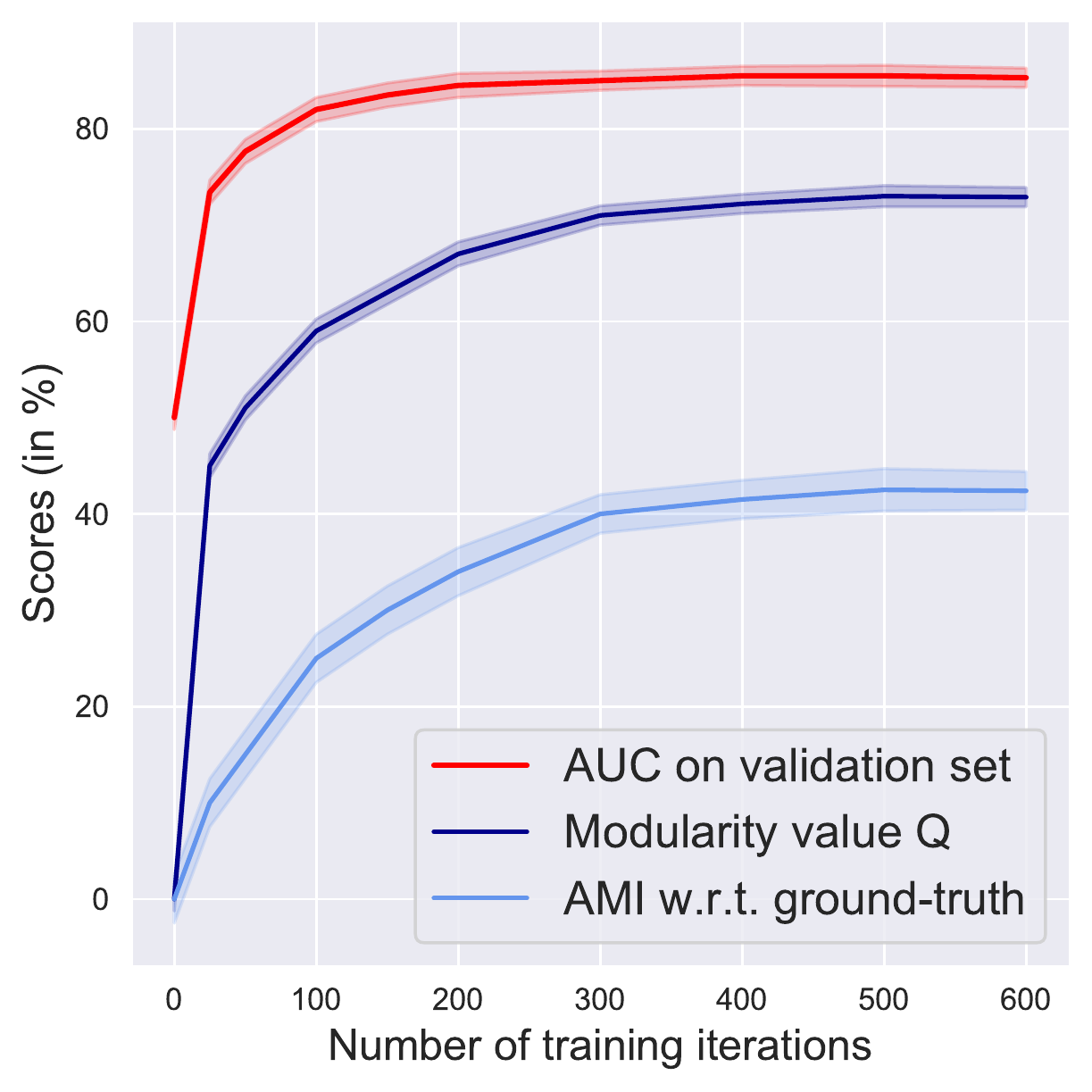}}}\subfigure[Pubmed]{
  \scalebox{0.47}{\includegraphics{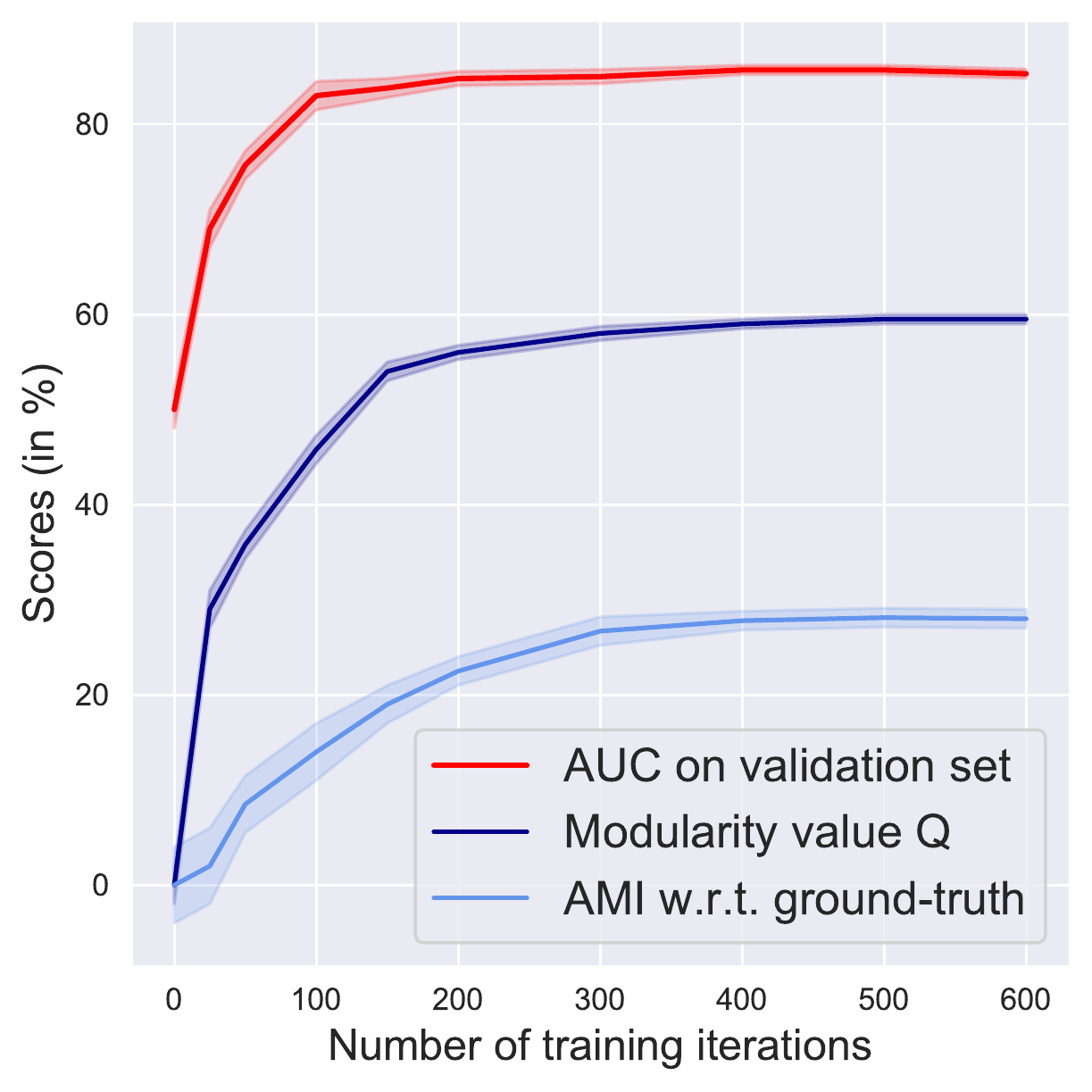}}}}
  \caption{Identification of the required number of training iterations, for Modularity-Aware VGAE with linear encoders trained on the featureless (a) Cora, and (b) Pubmed graphs. The plots report the evolution of the modularity $Q$ (\textcolor{Blue}{dark blue}) and AUC link prediction scores on validation sets (\textcolor{red}{red}) w.r.t. the number of model training iterations in gradient descent. By looking at the red curves only, one might choose to stop training models after 200 iterations as in \cite{kipf2016-2}, as the AUC validation scores have almost stabilized. However, the dark blue curves emphasize that $Q$ still increases up to 400-500 training iterations for both graphs. By also using $Q$ for hyperparameter selection (as we proposed), one will therefore continue training VGAE models up to 400-500 iterations. The \textcolor{MidnightBlue}{light blue} curves confirm that such a strategy eventually leads to better AMI final scores w.r.t. ground-truth communities. Note, that the light blue curves could \textit{not} be directly used for tuning, as ground-truth communities are assumed to be unavailable at training time.}
  \label{fig:optimization}
\end{figure*}

Figure~\ref{fig:optimization} provides an example of this phenomenon, for the number of training iterations required to train Modularity-Aware VGAE models on the featureless Cora and Pubmed graphs. The figure shows that, unlike our proposed procedure jointly based on AUC and $Q$, a hyperparameter selection based solely on AUC validation scores leads to earlier stopping of the model training and \textit{suboptimal} performances on community detection. This reaffirms the empirical relevance of our proposed procedure, and that optimal hyperparameters for joint link prediction and community detection might differ from those for link prediction only. Moreover, we note that, while Figure~\ref{fig:optimization} focuses on Modularity-Aware VGAE, our procedure also leads to the selection of a larger number of training iterations for the other GAE/VGAE-based methods under consideration in this work (values are similar to those in Table~\ref{tab:hyperparameterstable}), which explains why, on some occasions, we will report slightly improved results w.r.t. those obtained in the original papers.

\subsubsection{Results for Community Detection on Original Graphs (Task 1)}
\label{s422}

\begin{table}[t]
\begin{center}
\begin{small}
\centering
\caption{Results for Task 1 and Task 2 on the featureless Cora graph, using Modularity-Aware GAE and VGAE with Linear and GCN encoders, their standard GAE and VGAE counterparts, and other baselines. All node embedding models learn embedding vectors of dimension $d =16$, with other hyperparameters set as described in Section~\ref{s413}. Scores are averaged over 100 runs. For Task 2, link prediction results are reported from test sets (edges masked for the original graph in addition to the same number of randomly picked unconnected node pairs). \textbf{Bold} numbers correspond to the best performance for each score. Scores \textit{in italic} are within one standard deviation range from the best score.}
    \label{tab:coraresults}
   \resizebox{1.0\textwidth}{!}{
\begin{tabular}{r||cc||cc|cc}
\toprule
\textbf{Models} &  \multicolumn{2}{c}{\textbf{Task 1: Community Detection}} & \multicolumn{4}{c}{\textbf{Task 2: Joint Link Prediction and Community Detection}}\\
(Dimension $d=16$) & \multicolumn{2}{c}{\textbf{on complete graph}} & \multicolumn{4}{c}{\textbf{on graph with 15\% of edges being masked}}\\
\midrule
 & \textbf{AMI (in \%)} & \textbf{ARI (in \%)} &  \textbf{AMI (in \%)} &  \textbf{ARI (in \%)} &  \textbf{AUC (in \%)} & \textbf{AP (in \%)} \\ 
\midrule
\midrule
\underline{\textit{Modularity-Aware GAE/VGAE Models}} &  &  &  &  &  &  \\
Linear Modularity-Aware VGAE & \textbf{46.65} $\pm$ \textbf{0.94} & \textit{39.43} $\pm$ \textit{1.15} & \textit{42.86} $\pm$ \textit{1.65} & \textit{34.53} $\pm$ \textit{1.97} & \textit{85.96} $\pm$ \textit{1.24} & \textit{87.21} $\pm$ \textit{1.39} \\
Linear Modularity-Aware GAE & \textit{46.58} $\pm$ \textit{0.40} & \textbf{39.71} $\pm$ \textbf{0.41} & \textbf{43.48} $\pm$ \textbf{1.12} & \textbf{35.51} $\pm$ \textbf{1.20} & \textbf{87.18} $\pm$ \textbf{1.05} & \textit{88.53} $\pm$ \textit{1.33} \\
GCN-based Modularity-Aware VGAE & 43.25 $\pm$ 1.62 & 35.08 $\pm$ 1.88 & 41.03 $\pm$ 1.55 & \textit{33.43} $\pm$ \textit{2.17} & 84.87 $\pm$ 1.14 & 85.16 $\pm$ 1.23  \\
GCN-based Modularity-Aware GAE & 44.39 $\pm$ 0.85 & 38.70 $\pm$ 0.94 & 41.13 $\pm$ 1.35 & \textit{35.01} $\pm$ \textit{1.58} & \textit{86.90} $\pm$ \textit{1.16} & \textit{87.55} $\pm$ \textit{1.26} \\
\midrule
\underline{\textit{Standard GAE/VGAE Models}} &  &  &  &  &  &  \\
Linear VGAE & 37.12 $\pm$ 1.46 & 26.83 $\pm$ 1.68 & 32.22 $\pm$ 1.76 & 21.82 $\pm$ 1.80 & 85.69 $\pm$ 1.17 & \textbf{89.12} $\pm$ \textbf{0.82} \\
Linear GAE & 35.05 $\pm$ 2.55 & 24.32 $\pm$ 2.99 & 28.41 $\pm$ 1.68 & 19.45 $\pm$ 1.75 & 84.46 $\pm$ 1.64 & \textit{88.42} $\pm$ \textit{1.07} \\
GCN-based VGAE & 34.36 $\pm$ 3.66 & 23.98 $\pm$ 5.01 & 28.62 $\pm$ 2.76 & 19.70 $\pm$ 3.71 & 85.47 $\pm$ 1.18 & \textit{88.90} $\pm$ \textit{1.11} \\
GCN-based GAE & 35.64 $\pm$ 3.67 & 25.33 $\pm$ 4.06 & 31.30 $\pm$ 2.07 & 19.89 $\pm$ 3.07 & 85.31 $\pm$ 1.35 & \textit{88.67} $\pm$ \textit{1.24} \\
\midrule
\midrule
\underline{\textit{Other Baselines}} &  &  &  &  &  &  \\
Louvain & 42.70 $\pm$ 0.65 & 24.01 $\pm$ 1.70 & 39.09 $\pm$ 0.73 & 20.19 $\pm$ 1.73 & -- & -- \\
VGAECD& 36.11 $\pm$ 1.07 & 27.15 $\pm$ 2.05 & 33.54 $\pm$ 1.46 & 24.32 $\pm$ 2.25 & 83.12 $\pm$ 1.11 & 84.68 $\pm$ 0.98 \\
VGAECD-OPT & 38.93 $\pm$ 1.21 & 27.61 $\pm$ 1.82 & 34.41 $\pm$ 1.62 & 24.66 $\pm$ 1.98 & 82.89 $\pm$ 1.20 & 83.70 $\pm$ 1.16 \\
ARGVA & 34.97 $\pm$ 3.01 & 23.29 $\pm$ 3.21 & 28.96 $\pm$ 2.64 & 19.74 $\pm$ 3.02 & 85.85 $\pm$ 0.87 & \textit{88.94} $\pm$ \textit{0.72} \\
ARGA & 35.91 $\pm$ 3.11 & 25.88 $\pm$ 2.89 & 31.61 $\pm$ 2.05 & 20.18 $\pm$ 2.92 & 85.95 $\pm$ 0.85 & \textit{89.07} $\pm$ \textit{0.70} \\
DVGAE & 35.02 $\pm$ 2.73 & 25.03 $\pm$ 4.32 & 30.46 $\pm$ 4.12 & 21.06 $\pm$ 5.06 & 85.58 $\pm$ 1.31 & \textit{88.77} $\pm$ \textit{1.29} \\
DeepWalk & 36.58 $\pm$ 1.69 & 27.92 $\pm$ 2.93 & 30.26 $\pm$ 2.32 & 20.24 $\pm$ 3.91 & 80.67 $\pm$ 1.50 & 80.48 $\pm$ 1.28 \\
node2vec & 41.64 $\pm$ 1.25 & 34.30 $\pm$ 1.92 & 36.25 $\pm$ 1.38 & 29.43 $\pm$ 2.21 & 82.43 $\pm$ 1.23 & 81.60 $\pm$ 0.91 \\
\bottomrule
\end{tabular}}
\end{small}
\end{center}
\end{table}

\begin{figure*}[t]
\centering
\resizebox{1.03\textwidth}{!}{
  \subfigure[Linear Standard VGAE]{
  \scalebox{0.48}{\includegraphics{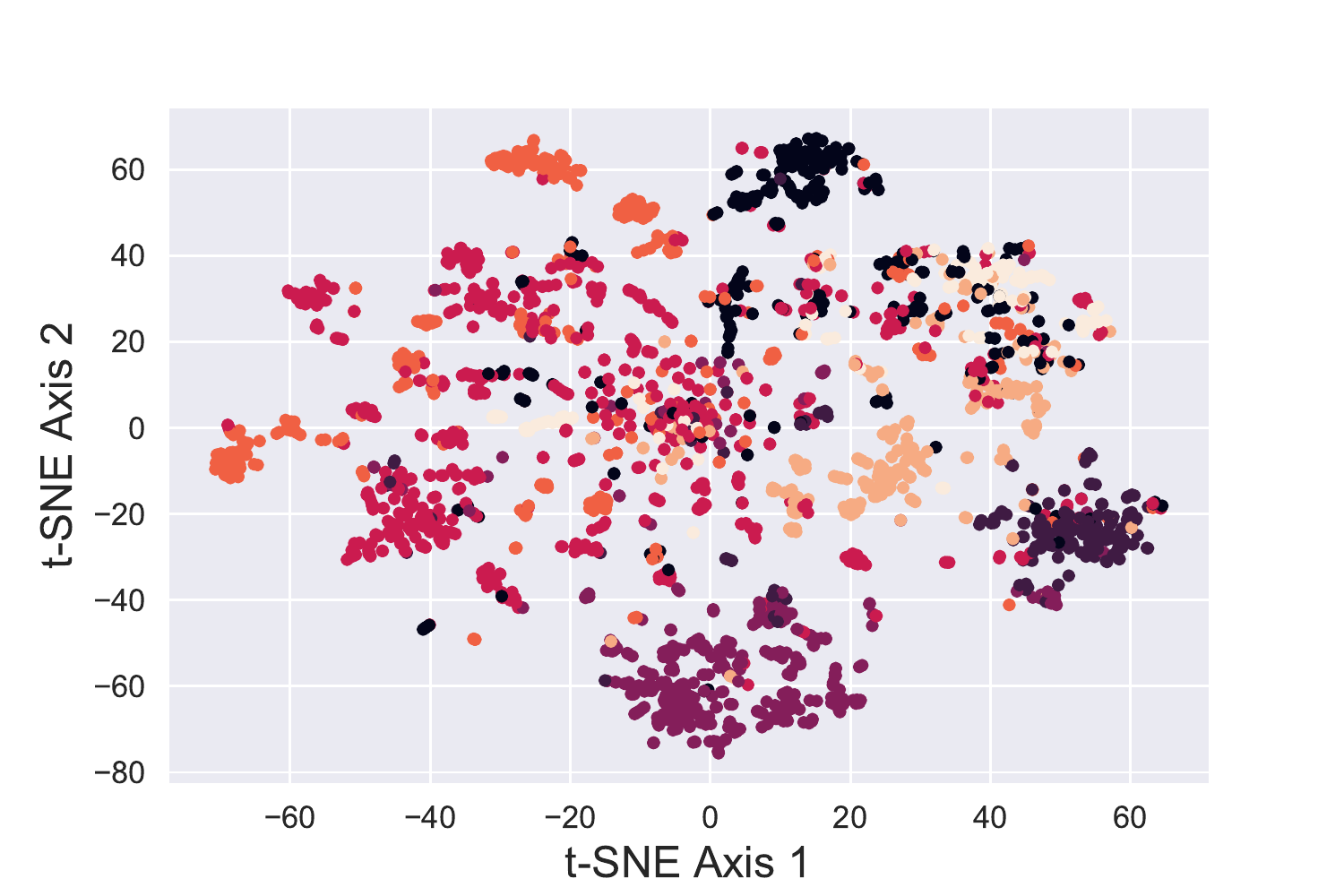}}}\subfigure[Linear Modularity-Aware VGAE]{
  \scalebox{0.48}{\includegraphics{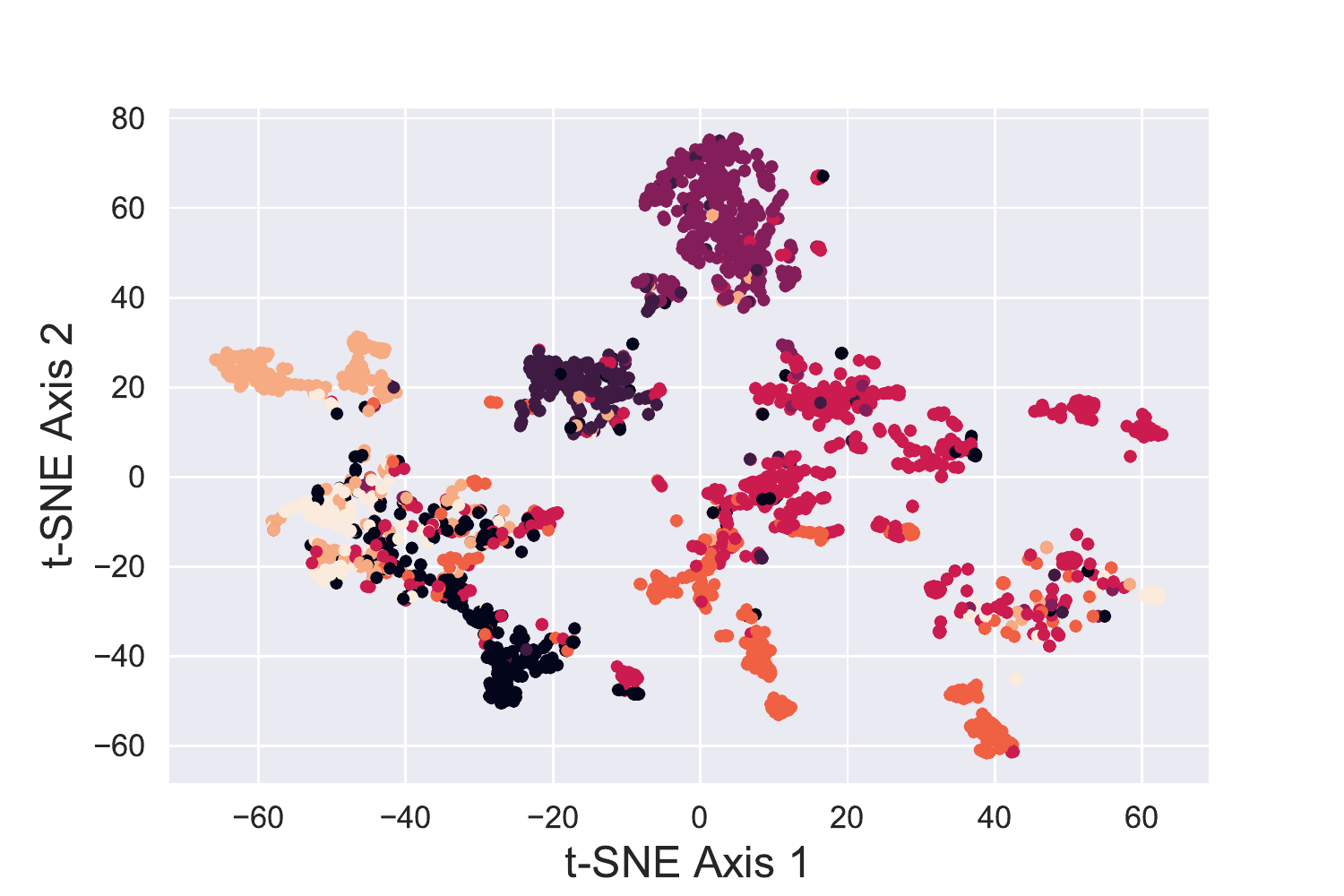}}}}
  \caption{Visualization of node embedding representations for the featureless Cora graph, learned by (a)~Standard VGAE, and (b)~Modularity-Aware VGAE, with linear encoders. The plots were obtained using the t-SNE method for high-dimensional data visualization. Colors denote ground-truth communities, that were not available during training. Although community detection is not perfect (both methods return AMI scores $<$ 50\% in Table~\ref{tab:coraresults}), node embedding representations from (b) provide a more visible separation of these communities. Specifically, in Table~\ref{tab:coraresults}, using  Linear Modularity-Aware VGAE for community detection leads to an increase of 9 AMI points (Task 1) to 10 AMI points (Task 2) for community detection w.r.t. Linear Standard VGAE, while preserving comparable performances on link prediction (Task 2).}
  \label{visucora}
\end{figure*}

We now focus on the ``pure'' community detection task (Task 1), performed by models trained on graphs, where no edges are removed for model training as previously introduced in Section~\ref{s412}. The second and third column of Table~\ref{tab:coraresults} reports mean AMI and ARI scores on Cora for this task along with standard deviations over 100 runs, for Modularity-Aware GAE and VGAE models (with linear or GCN encoders), their standard counterparts and~other~baselines. 

We draw several conclusions from Table \ref{tab:coraresults}. Foremost, previous conclusions~\cite{choong2018learning,salha2019-1,choong2020optimizing,salha2021fastgae} on the limitations of standard GAE and VGAE for community detection are confirmed: in Table~\ref{tab:coraresults}, these methods are notably outperformed by a direct use of the Louvain method (e.g., 42.70\% vs 34.36\% mean AMI scores for Louvain vs GCN-based VGAE). We also observe that previous GAE/VGAE extensions, reported as baselines, actually provide few empirical benefits w.r.t. standard GAE and VGAE for this \textit{featureless} graph (e.g., only +1.81 AMI points for VGAECD-OPT\footnote{The increase is even smaller when replacing AMI scores obtained via our re-implementation of VGAECD and VGAECD-OPT (i.e., 36.11\% and 38.93\%, respectively) by AMI scores originally reported in \cite{choong2020optimizing} for these methods (i.e., 28.22\% and 37.35\%, respectively), which are lower than ours.} vs Linear VGAE). Such a result, in conjunction with the improved performances of these same baselines on graphs \textit{with features} (see thereafter), tends to confirm our initial diagnosis that various GAE/VGAE extensions for community detection mainly benefit from the presence of node features.

On the contrary, our proposed Modularity-Aware GAE and VGAE models, incorporating Louvain clusters as a prior signal in the GAE's and VGAE's encoders, significantly outperform both the use of the Louvain method alone, and the use of GAE and VGAE alone (e.g., with a top 46.65\% mean AMI for Modularity-Aware VGAE with linear encoders, and a top 39.71\% mean ARI score for Modularity-Aware GAE with linear encoders). Modularity-Aware models also compare favorably to the baselines under consideration (e.g., with +12.1 ARI points for Linear Modularity-Aware GAE w.r.t. VGAECD-OPT), while also providing less volatile results w.r.t. standard GAE/VGAE.  Furthermore, we note that Modularity-Aware models with linear encoders tend to outperform their GCN-based counterparts (consistently with previous findings from \cite{salha2020simple,choong2020optimizing} on Cora) and that GAE and VGAE reach comparable scores. In addition to these results, Figure~\ref{visucora} visualizes the node embeddings learned by our models using t-SNE\footnote{We used the scikit-learn \cite{pedregosa2011scikit} implementation of this data visualization method: \href{https://scikit-learn.org/stable/modules/generated/sklearn.manifold.TSNE.html}{https://scikit-learn.org/stable/modules/generated/sklearn.manifold.TSNE.html} (accessed October 13, 2021).}~\cite{van2008visualizing}.

Overall, we obtain similar conclusions on the other graph datasets. Following the format of Table~\ref{tab:coraresults}, columns two and three of Table~\ref{tab:pubmedresults} present detailed community detection results for the featureless Pubmed graph. Table~\ref{tab:allresults} reports more summarized results for all other graph datasets under consideration, with and without node features (when available).
While Louvain outperforms standard GAE/VGAE in 5 featureless graphs out of 7 in Table~\ref{tab:allresults} (e.g., 19.81\% vs 15.79\% mean AMI scores for Louvain vs GCN-based VGAE on Deezer-Album), our Modularity-Aware models manage to achieve either comparable or better performances w.r.t. standard models, Louvain and other baselines in the wide majority of experiments. Furthermore, throughout Table~\ref{tab:allresults}, we observe that linear encoders outperform their GCN-based counterparts in 8/10 experiments, and that VGAE models outperform GAE models in 8/10 experiments (even though performances are often relatively close, as for Cora). Moreover we emphasize that, while all tables report results for fixed embedding dimensions of $d = 16$, we reached similar conclusions for $d \in \{32, 64\}$. Although performances sometimes improved by increasing $d$, the \textit{ranking} of methods under consideration remained similar: for instance, by setting $d = 64$, Deepwalk's mean AMI score increased from 36.58\% to roughly 41\% on the featureless Cora graph, while the mean AMI score from our Linear Modularity-Aware GAE simultaneously increased from 46.58\% to 47.80\%. 
Lastly, while we observed that our results were quite sensitive to our choice of hyperparameters $\lambda$, $\beta$, and $\gamma$, no direct link with the number of nodes, of edges, or of ground-truth communities in the graph seems to emerge from our experiments.

\begin{table}[t]
\begin{center}
\begin{small}
\centering
\caption{Results for Task 1 and Task 2 on the featureless Pubmed graph, using Modularity-Aware GAE and VGAE with Linear and GCN encoders, their standard GAE and VGAE counterparts, and other baselines. All node embedding models learn embedding vectors of dimension $d =16$, with other hyperparameters set as described in Section~\ref{s413}. Scores are averaged over 100 runs. For Task 2, link prediction results are reported from test sets (edges masked during training + same number of randomly picked unconnected node pairs).  \textbf{Bold} numbers correspond to the best performance for each score. Scores \textit{in italic} are within one standard deviation range from the best score.}
    \label{tab:pubmedresults}
   \resizebox{1.0\textwidth}{!}{
\begin{tabular}{r||cc||cc|cc}
\toprule
\textbf{Models} &  \multicolumn{2}{c}{\textbf{Task 1: Community Detection}} & \multicolumn{4}{c}{\textbf{Task 2: Joint Link Prediction and Community Detection}}\\
(Dimension $d=16$) & \multicolumn{2}{c}{\textbf{on complete graph}} & \multicolumn{4}{c}{\textbf{on graph with 15\% of edges being masked}}\\
\midrule
 & \textbf{AMI (in \%)} & \textbf{ARI (in \%)} &  \textbf{AMI (in \%)} &  \textbf{ARI (in \%)} &  \textbf{AUC (in \%)} & \textbf{AP (in \%)} \\ 
\midrule
\midrule
\underline{Modularity-Aware GAE/VGAE Models} &  &  &  &  &  &  \\
Linear Modularity-Aware VGAE & 28.12 $\pm$ 0.29 & 29.01 $\pm$ 0.51 & 25.93 $\pm$ 0.65 & 23.76 $\pm$ 0.49 & \textbf{85.76} $\pm$ \textbf{0.37} & 87.77 $\pm$ 0.31 \\
Linear Modularity-Aware GAE & \textit{28.54} $\pm$ \textit{0.24} & 26.36 $\pm$ 0.34 & \textbf{26.38} $\pm$ \textbf{0.43} & 21.30 $\pm$ 0.59 & 84.39 $\pm$ 0.32 & \textit{87.92} $\pm$ \textit{0.40} \\
GCN-based Modularity-Aware VGAE & 28.08 $\pm$ 0.27 & 28.14 $\pm$ 0.33 & 25.70 $\pm$ 0.86 & 22.65 $\pm$ 0.80 & 84.70 $\pm$ 0.24 & 86.64 $\pm$ 0.15 \\
GCN-based Modularity-Aware GAE & \textbf{28.74} $\pm$ \textbf{0.28} & 26.71 $\pm$ 0.47 & 25.52 $\pm$ 0.45 & 20.52 $\pm$ 0.31 & 85.07 $\pm$ 0.35 & \textit{88.27} $\pm$ \textit{0.39} \\
\midrule
\underline{Standard GAE/VGAE Models} &  &  &  &  &  &  \\
Linear VGAE & 22.16 $\pm$ 2.02 & 13.90 $\pm$ 3.47  & 21.78 $\pm$ 2.57 & 13.81 $\pm$ 3.17 & 84.57 $\pm$ 0.51 & \textbf{88.31} $\pm$ \textbf{0.44} \\
Linear GAE & 12.61 $\pm$ 4.61 & 6.37 $\pm$ 3.86  & 12.60 $\pm$ 4.67 & 6.21 $\pm$ 1.75 & 82.03 $\pm$ 0.32 & 87.71 $\pm$ 0.24 \\
GCN-based VGAE & 20.11 $\pm$ 3.05 & 13.12 $\pm$ 3.10 & 17.34 $\pm$ 2.99 & 8.71 $\pm$ 3.05 & 82.19 $\pm$ 0.88 & 87.51 $\pm$ 0.55 \\
GCN-based GAE & 20.12 $\pm$ 2.89 & 14.21 $\pm$ 2.78 & 16.75 $\pm$ 3.36 & 9.18 $\pm$ 2.71 & 82.33 $\pm$ 1.32 & 87.20 $\pm$ 0.58 \\

\midrule
\midrule
\underline{Other Baselines} &  &  &  &  &  &  \\
Louvain & 20.06 $\pm$ 0.27 & 10.34 $\pm$ 0.99 & 16.71 $\pm$ 0.46 & 8.32 $\pm$ 0.79 & -- & -- \\
VGAECD & 20.32 $\pm$ 2.95 & 13.54 $\pm$ 2.98 & 17.39 $\pm$ 3.04 & 9.21 $\pm$ 3.12 & 82.05 $\pm$ 0.90 & 87.30 $\pm$ 0.53 \\
VGAECD-OPT & 22.50 $\pm$ 1.99 & 14.58 $\pm$ 2.86 & 21.98 $\pm$ 2.46 & 15.22 $\pm$ 2.92 & 82.03 $\pm$ 0.82 &  87.41 $\pm$ 0.53 \\
ARGVA & 20.73 $\pm$ 3.10 & 13.94 $\pm$ 3.12 & 17.63 $\pm$ 3.19 & 9.19 $\pm$ 3.09 & 84.07 $\pm$ 0.55 & 87.73 $\pm$ 0.49 \\
ARGA & 20.98 $\pm$ 2.90 & 14.79 $\pm$ 2.80 & 17.21 $\pm$ 3.01 & 9.59 $\pm$ 2.76 & 83.73 $\pm$ 0.53 & \textit{87.90} $\pm$ \textit{0.45} \\
DVGAE & 23.15 $\pm$ 2.52 & 15.02 $\pm$ 3.33 & 22.10 $\pm$ 2.50 & 14.62 $\pm$ 2.96 & 83.21 $\pm$ 0.92 & \textit{88.17} $\pm$ \textit{0.49} \\
DeepWalk & \textit{28.53} $\pm$ \textit{0.43} & 29.61 $\pm$ 0.33 & 15.80 $\pm$ 1.05 & 16.16 $\pm$ 1.75 & 80.63 $\pm$ 0.42 & 81.03 $\pm$ 0.54 \\
node2vec & \textit{28.52} $\pm$ \textit{1.12} & \textbf{30.63} $\pm$ \textbf{1.14} & 23.88 $\pm$ 0.54 & \textbf{25.90} $\pm$ \textbf{0.65} & 81.03 $\pm$ 0.30 & 82.33 $\pm$ 0.41 \\
\bottomrule
\end{tabular}}
\end{small}
\end{center}
\end{table}

Interestingly, we also observe that combining the Louvain method and GAE/VGAE in our Modularity-Aware models might be empirically beneficial \textit{even when standard GAE/VGAE initially outperform the Louvain method}. For instance in Table~\ref{tab:pubmedresults}, our Linear Modularity-Aware VGAE outperforms Linear Standard VGAE (e.g., with 28.12\% vs 22.16\% mean AMI scores), despite the fact that this standard model initially outperformed the Louvain method (20.06\% mean AMI score). This tends to confirm that modularity-based clustering \textit{à la} Louvain complements the encoding-decoding paradigm of GAE and VGAE, and that learning node embedding spaces from complementary criteria is empirically beneficial. On a more negative note, we nonetheless acknowledge that, on Cora, Citeseer and Pubmed in Table~\ref{tab:allresults}, empirical gains of Modularity-Aware models are less visible on graphs \textit{equipped with node features} than on featureless graphs, which we will further discuss in our limitation section.

As our Modularity-Aware models include two main novel components (namely our community-preserving encoders from Section~\ref{s32} and our revised loss from Section~\ref{s33}), one might wonder what the contribution of each of these components to the performance gains is. To study this question, we report in Figure~\ref{fig:ablation} the results of an \textit{ablation study}, that consisted in training variant versions of our models leveraging one of these components only\footnote{A Modularity-Aware GAE or VGAE model that leverages the novel encoder only (respectively, the novel loss only) corresponds to a particular case of a ``complete'' Modularity-Aware GAE or VGAE model, where the hyperparameter $\beta$ (respectively, the hyperparameter $\lambda$) is set~to~0.} (i.e., the encoder but not the loss, or the loss but not the encoder). Figure~\ref{fig:ablation} shows that incorporating any of these two individual components into the model improves community detection. The gain is larger for the \textit{loss} in the Cora example from Figure~\ref{fig:ablation}(a), while it is larger for the \textit{encoder} in the Deezer-Album example from Figure~\ref{fig:ablation}(b). A simultaneous use of the encoder and of the loss leads to the best results in both examples, which we also confirmed on the other graphs under consideration.

\begin{table}[t]
\begin{center}
\begin{small}
\centering
\caption{Summarized results for Task 1 and Task 2 on all graphs. For each graph, for brevity, we only report the \textbf{best} Modularity-Inspired model (best on Task 2, among GCN \textbf{or} Linear encoder, and GAE \textbf{or} VGAE), its standard counterpart, and a comparison to the Louvain baseline as well as the best other baseline (among VGAECD, VGAECD-OPT, ARGA, ARGVA, DVGAE, DeepWalk and node2vec). All node embedding models learn embedding vectors of dimension $d =16$, with other hyperparameters set as described in Section~\ref{s413}. Scores are averaged over 100 runs except for the larger SBM and Deezer-Album graphs (10 runs). \textbf{Bold} numbers correspond to the best performance for each score. Scores \textit{in italic} are within one standard deviation range from the best score.}
    \label{tab:allresults}
     \vspace{0.2cm}
   \resizebox{1.0\textwidth}{!}{
\begin{tabular}{c|r||cc||cc|cc}
\toprule
\textbf{Datasets} & \textbf{Models} &  \multicolumn{2}{c}{\textbf{Task 1: Community Detection}} & \multicolumn{4}{c}{\textbf{Task 2: Joint Link Prediction and Community Detection}}\\
& (Dimension $d=16$) & \multicolumn{2}{c}{\textbf{on complete graph}} & \multicolumn{4}{c}{\textbf{on graph with 15\% of edges being masked}}\\
\midrule
&  & \textbf{AMI (in \%)} & \textbf{ARI (in \%)} &  \textbf{AMI (in \%)} &  \textbf{ARI (in \%)} &  \textbf{AUC (in \%)} & \textbf{AP (in \%)} \\ 
\midrule
\midrule
& GCN-based Modularity-Aware VGAE & \textbf{73.74} $\pm$ \textbf{1.32} & \textbf{82.78} $\pm$ \textbf{1.27} & \textbf{70.42} $\pm$ \textbf{1.28} & \textbf{79.80} $\pm$ \textbf{1.12} & \textbf{91.67} $\pm$ \textbf{0.39} & \textit{92.37} $\pm$ \textit{0.41} \\
& GCN-based Standard VGAE & \textit{73.42} $\pm$ \textit{0.95} & \textit{82.58} $\pm$ \textit{0.93} & 66.90 $\pm$ 3.32 & \textit{77.23} $\pm$ \textit{3.89} & \textit{91.64} $\pm$ \textit{0.42} & \textbf{92.52} $\pm$ \textbf{0.51} \\ 
\textbf{Blogs} & Louvain & 63.43 $\pm$ 0.86 & 76.66 $\pm$ 0.70 & 57.25 $\pm$ 1.67 & 73.00 $\pm$ 1.56 & -- & -- \\
& \underline{Best other baseline:} &  &  &  &  &  & \\
& node2vec & 72.88 $\pm$ 0.87 & 82.08 $\pm$ 0.73 & 67.64 $\pm$ 1.23 & 77.03 $\pm$ 1.85 & 83.63 $\pm$ 0.34 & 79.60 $\pm$ 0.61\\
\midrule
& Linear Modularity-Aware GAE & \textbf{46.58} $\pm$ \textbf{0.40} & \textbf{39.71} $\pm$ \textbf{0.41} & \textbf{43.48} $\pm$ \textbf{1.12} & \textbf{35.51} $\pm$ \textbf{1.20} & \textbf{87.18} $\pm$ \textbf{1.05} & \textbf{88.53} $\pm$ \textbf{1.33} \\
& Linear Standard GAE & 35.05 $\pm$ 2.55 & 24.32 $\pm$ 2.99 & 28.41 $\pm$ 1.68 & 19.45 $\pm$ 1.75 & 84.46 $\pm$ 1.64 & \textit{88.42} $\pm$ \textit{1.07} \\
\textbf{Cora} & Louvain & 42.70 $\pm$ 0.65 & 24.01 $\pm$ 1.70 & 39.09 $\pm$ 0.73 & 20.19 $\pm$ 1.73 & -- & -- \\
& \underline{Best other baseline:} &  & &  &  & &  \\
& node2vec & 41.64 $\pm$ 1.25 & 34.30 $\pm$ 1.92 & 36.25 $\pm$ 1.38 & 29.43 $\pm$ 2.21 & 82.43 $\pm$ 1.23 & 81.60 $\pm$ 0.91 \\
\midrule
& Linear Modularity-Aware VGAE & \textbf{52.43} $\pm$ \textbf{1.87} & \textbf{44.82} $\pm$ \textbf{3.12} & \textbf{49.48} $\pm$ \textbf{2.15} & \textbf{43.05} $\pm$ \textbf{3.51} & \textbf{93.10} $\pm$ \textbf{0.88} & \textbf{94.06} $\pm$ \textbf{0.75} \\
\textbf{Cora} & Linear Standard VGAE & 49.98 $\pm$ 2.40 & \textit{43.15} $\pm$ \textit{4.35} & 46.90 $\pm$ 1.43 & 38.24 $\pm$ 3.56 & \textit{93.04} $\pm$ \textit{0.80} & \textit{94.04} $\pm$ \textit{0.75} \\ 
\textbf{with} & Louvain & 42.70 $\pm$ 0.65 & 24.01 $\pm$ 1.70 & 39.09 $\pm$ 0.73 & 20.19 $\pm$ 1.73 & -- & -- \\
\textbf{features} & \underline{Best other baseline:} &  &  &  &  &  & \\
& VGAECD-OPT & 50.32\tablefootnote{\label{snap} We note that authors of VGAECD-OPT \cite{choong2020optimizing} reported larger AMI scores w.r.t. those we managed to obtain from our re-implementation on Cora and Pubmed with features: 54.37\% and 35.52\%, respectively.} $\pm$ 1.95 & \textit{43.54} $\pm$ \textit{3.23} & 47.83 $\pm$ 1.64 & 39.45 $\pm$ 3.53 & \textit{92.25} $\pm$ \textit{1.07} & 92.60 $\pm$ 0.91 \\
\midrule

& Linear Modularity-Aware VGAE & 21.28 $\pm$ 1.03 & \textbf{15.39} $\pm$ \textbf{1.06} & 19.05 $\pm$ 1.47 & \textbf{12.19} $\pm$ \textbf{1.38} & \textbf{80.84} $\pm$ \textbf{1.64} & \textbf{84.21} $\pm$ \textbf{1.21} \\
& Linear Standard VGAE & 13.83 $\pm$ 1.00 & 8.31 $\pm$ 0.89 & 11.11 $\pm$ 1.10 & 5.87 $\pm$ 0.87 & 78.26 $\pm$ 1.55 & \textit{82.93} $\pm$ \textit{1.39} \\ 
\textbf{Citeseer} & Louvain & \textbf{24.72} $\pm$ \textbf{0.27} & 9.21 $\pm$ 0.75 & \textbf{22.71} $\pm$ \textbf{0.47} & 7.70 $\pm$ 0.67 & -- & -- \\
& \underline{Best other baseline:} &  &  &  &  &  & \\
& node2vec & 18.68 $\pm$ 1.13  & \textit{14.93} $\pm$ \textit{1.15} & 14.40 $\pm$ 1.18 & \textit{12.13} $\pm$ \textit{1.53} & 76.05 $\pm$ 2.12 & 79.46 $\pm$ 1.65 \\
\midrule

& Linear Modularity-Aware VGAE  & \textbf{25.11} $\pm$ \textbf{0.94} & \textbf{15.55} $\pm$ \textbf{0.60} & \textit{22.21} $\pm$ \textit{1.24} & \textbf{12.59} $\pm$ \textbf{1.25} & 86.54 $\pm$ 1.20 & 88.07 $\pm$ 1.22 \\
\textbf{Citeseer} & Linear Standard VGAE & 17.80 $\pm$ 1.61 & 6.01 $\pm$ 1.46 & 17.38 $\pm$ 1.43 & 6.10 $\pm$ 1.51 & \textbf{89.08} $\pm$ \textbf{1.19} & \textbf{91.19} $\pm$ \textbf{0.98} \\ 
\textbf{with} & Louvain & 24.72 $\pm$ 0.27 & 9.21 $\pm$ 0.75 & \textbf{22.71} $\pm$ \textbf{0.47} & 7.70 $\pm$ 0.67 & -- & -- \\
\textbf{features} & \underline{Best other baseline:}&  &  &  &  &  & \\
& DVGAE & 20.09 $\pm$ 2.84 & 12.16 $\pm$ 2.74 & 16.02 $\pm$ 3.32 & \textit{10.03} $\pm$ \textit{4.48} & 86.85 $\pm$ 1.48 & 88.43 $\pm$ 1.23 \\
\midrule
& Linear Modularity-Aware GAE & \textbf{28.54} $\pm$ \textbf{0.24} & 26.36 $\pm$ 0.34 & \textbf{26.38} $\pm$ \textbf{0.43} & 21.30 $\pm$ 0.59 & \textbf{84.39} $\pm$ \textbf{0.32} & \textbf{87.92} $\pm$ \textbf{0.40} \\
& Linear Standard GAE & 12.61 $\pm$ 4.61 & 6.37 $\pm$ 3.86  & 12.60 $\pm$ 4.67 & 6.21 $\pm$ 1.75 & 82.03 $\pm$ 0.32 & \textit{87.71} $\pm$ \textit{0.24} \\
\textbf{Pubmed} & Louvain & 20.06 $\pm$ 0.27 & 10.34 $\pm$ 0.99 & 16.71 $\pm$ 0.46 & 8.32 $\pm$ 0.79 & -- & -- \\
& \underline{Best other baseline:} &  &  &  &  &  & \\
& node2vec & \textit{28.52} $\pm$ \textit{1.12} & \textbf{30.63} $\pm$ \textbf{1.14} & 23.88 $\pm$ 0.54 & \textbf{25.90} $\pm$ \textbf{0.65} & 81.03 $\pm$ 0.30 & 82.33 $\pm$ 0.41 \\
\midrule
& Linear Modularity-Aware VGAE & 30.09 $\pm$ 0.63 & \textbf{29.11} $\pm$ \textbf{0.65} & \textbf{29.60} $\pm$ \textbf{0.70} & \textbf{28.54} $\pm$ \textbf{0.74} & \textit{97.10} $\pm$ \textit{0.21} & \textbf{97.21} $\pm$ \textbf{0.18} \\
\textbf{Pubmed} & Linear Standard VGAE & 29.98 $\pm$ 0.41 & \textit{29.05} $\pm$ \textit{0.20} & \textit{29.51} $\pm$ \textit{0.52} & \textit{28.50} $\pm$ \textit{0.36} & \textbf{97.12} $\pm$ \textbf{0.20} & \textit{97.20} $\pm$ \textit{0.17} \\ 
\textbf{with}  & Louvain & 20.06 $\pm$ 0.27 & 10.34 $\pm$ 0.99 & 16.71 $\pm$ 0.46 & 8.32 $\pm$ 0.79 & -- & -- \\
\textbf{features} & \underline{Best other baseline:} &  &  &  &  &  & \\
& VGAECD-OPT & \textbf{32.47}\textsuperscript{\ref{snap}} $\pm$ \textbf{0.45} & \textit{29.09} $\pm$ \textit{0.42} & \textit{29.46} $\pm$ \textit{0.52} & \textit{28.43} $\pm$ \textit{0.61} & 94.27 $\pm$ 0.33 & 94.53 $\pm$ 0.36 \\
\midrule
& Linear Modularity-Aware VGAE & \textbf{48.55} $\pm$ \textbf{0.18} & \textbf{22.21} $\pm$ \textbf{0.39} & \textbf{46.10} $\pm$ \textbf{0.29} & \textbf{20.24} $\pm$ \textbf{0.41} & \textbf{95.76} $\pm$ \textbf{0.17} & \textbf{96.31} $\pm$ \textbf{0.12} \\
& Linear Standard VGAE & 46.07 $\pm$ 0.54 & 20.01 $\pm$ 0.90 & 43.38 $\pm$ 0.37 & 18.02 $\pm$ 0.66 & \textit{95.55} $\pm$ \textit{0.22} & \textit{96.30} $\pm$ \textit{0.18} \\ 
\textbf{Cora-Larger} & Louvain & 44.72 $\pm$ 0.50 & 19.46 $\pm$ 0.66 & 43.41 $\pm$ 0.52 & 19.29 $\pm$ 0.68 & -- & -- \\
& \underline{Best other baseline:}&  &  &  &  &  & \\
& DVGAE & 46.63 $\pm$ 0.56 & 20.72 $\pm$ 0.96 & 43.48 $\pm$ 0.61 & 18.45 $\pm$ 0.67 & 94.97 $\pm$ 0.23 & 95.98 $\pm$ 0.21 \\
\midrule
& Linear Modularity-Aware VGAE & \textbf{36.02} $\pm$ \textbf{0.13} & \textbf{8.12} $\pm$ \textbf{0.06} & \textbf{35.85} $\pm$ \textbf{0.20} & \textbf{8.06} $\pm$ \textbf{0.11} & \textit{82.34} $\pm$ \textit{0.38} & \textit{86.76} $\pm$ \textit{0.41} \\
& Linear Standard VGAE  & 35.01 $\pm$ 0.21 & 7.88 $\pm$ 0.15 & 30.79 $\pm$ 0.21 & 6.50 $\pm$ 0.13 & 80.11 $\pm$ 0.35 & 83.40 $\pm$ 0.36 \\ 
\textbf{SBM} & Louvain & \textit{36.00} $\pm$ \textit{0.15} & \textit{8.10} $\pm$ \textit{0.15} & \textit{35.84} $\pm$ \textit{0.18} & \textit{8.03} $\pm$ \textit{0.09} & -- & -- \\
& \underline{Best other baseline:} &  &  &  &  &  & \\
& DVGAE & \textit{35.90} $\pm$ \textit{0.18} & \textit{8.07} $\pm$ \textit{0.15} & 35.53 $\pm$ 0.23 & \textit{7.95} $\pm$ \textit{0.19} & \textbf{82.59} $\pm$ \textbf{0.36} & \textbf{87.08} $\pm$ \textbf{0.40} \\
\midrule
& GCN-Based Modularity-Aware VGAE & \textbf{21.64} $\pm$ \textbf{0.18} & \textbf{13.19} $\pm$ \textbf{0.09} & \textbf{19.10} $\pm$ \textbf{0.21} & \textbf{12.00} $\pm$ \textbf{0.17} & \textbf{85.40} $\pm$ \textbf{0.14} & \textit{86.38} $\pm$ \textit{0.15} \\
& GCN-Based Standard VGAE & 15.79 $\pm$ 0.32 & 9.75 $\pm$ 0.21 & 13.98 $\pm$ 0.35 & 8.81 $\pm$ 0.32 & \textit{85.37} $\pm$ \textit{0.12} & \textbf{86.41} $\pm$ \textbf{0.11} \\ 
\textbf{Deezer-Album} & Louvain & 19.81 $\pm$ 0.19 & 12.21 $\pm$ 0.09 & 17.68 $\pm$ 0.20 & 11.02 $\pm$ 0.13 & -- & -- \\
& \underline{Best other baseline:}&  &  &  &  &  & \\
& node2vec & 20.03 $\pm$ 0.24 & 12.20 $\pm$ 0.19 & 18.34 $\pm$ 0.29 & 11.27 $\pm$ 0.28 & 83.51 $\pm$ 0.17 & 84.12 $\pm$ 0.15 \\
\bottomrule
\end{tabular}}
\end{small}
\end{center}
\end{table}

\begin{figure*}[t]
\centering
\resizebox{1.0\textwidth}{!}{
  \subfigure[Cora]{
  \scalebox{0.48}{\includegraphics{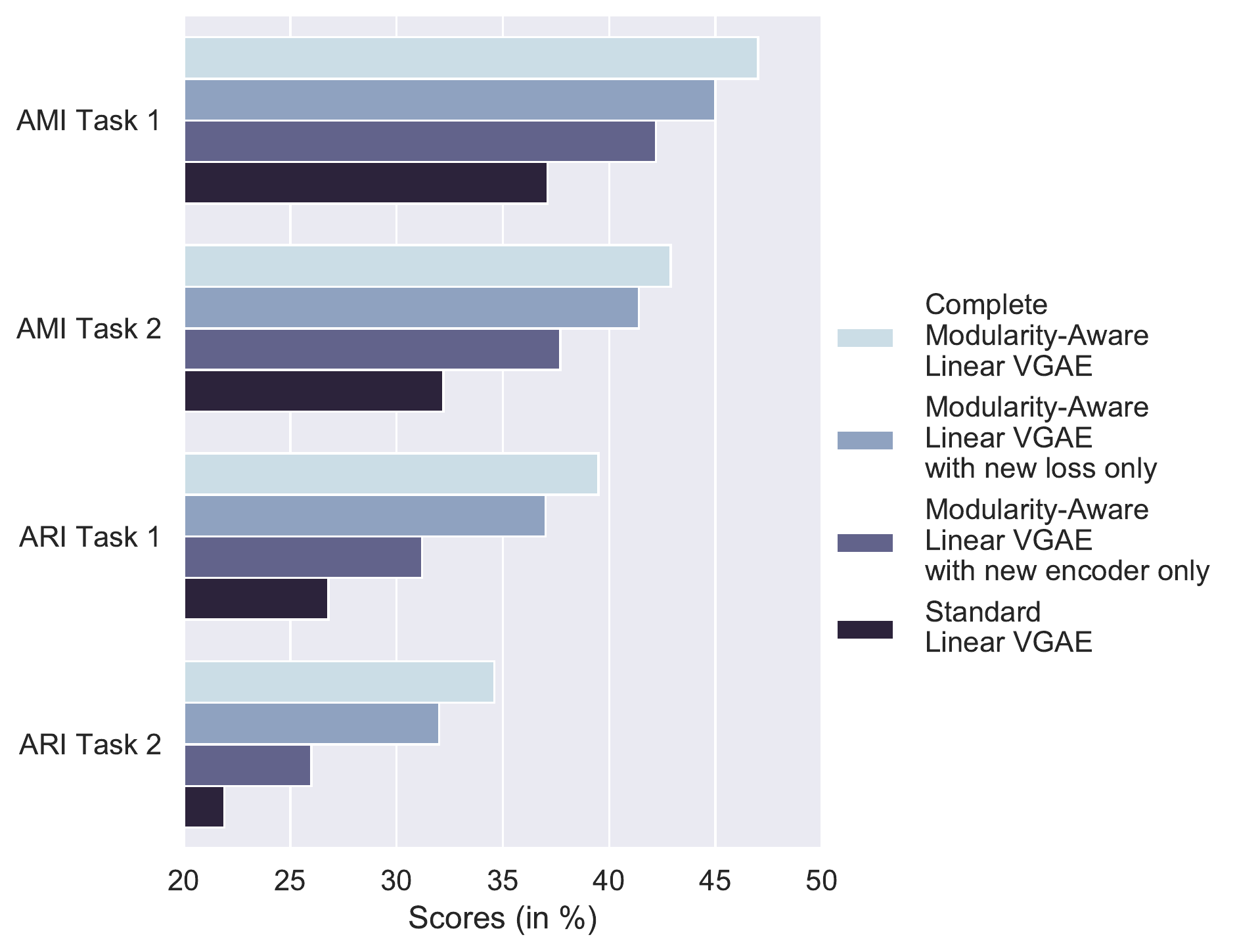}}}\subfigure[Deezer]{
  \scalebox{0.48}{\includegraphics{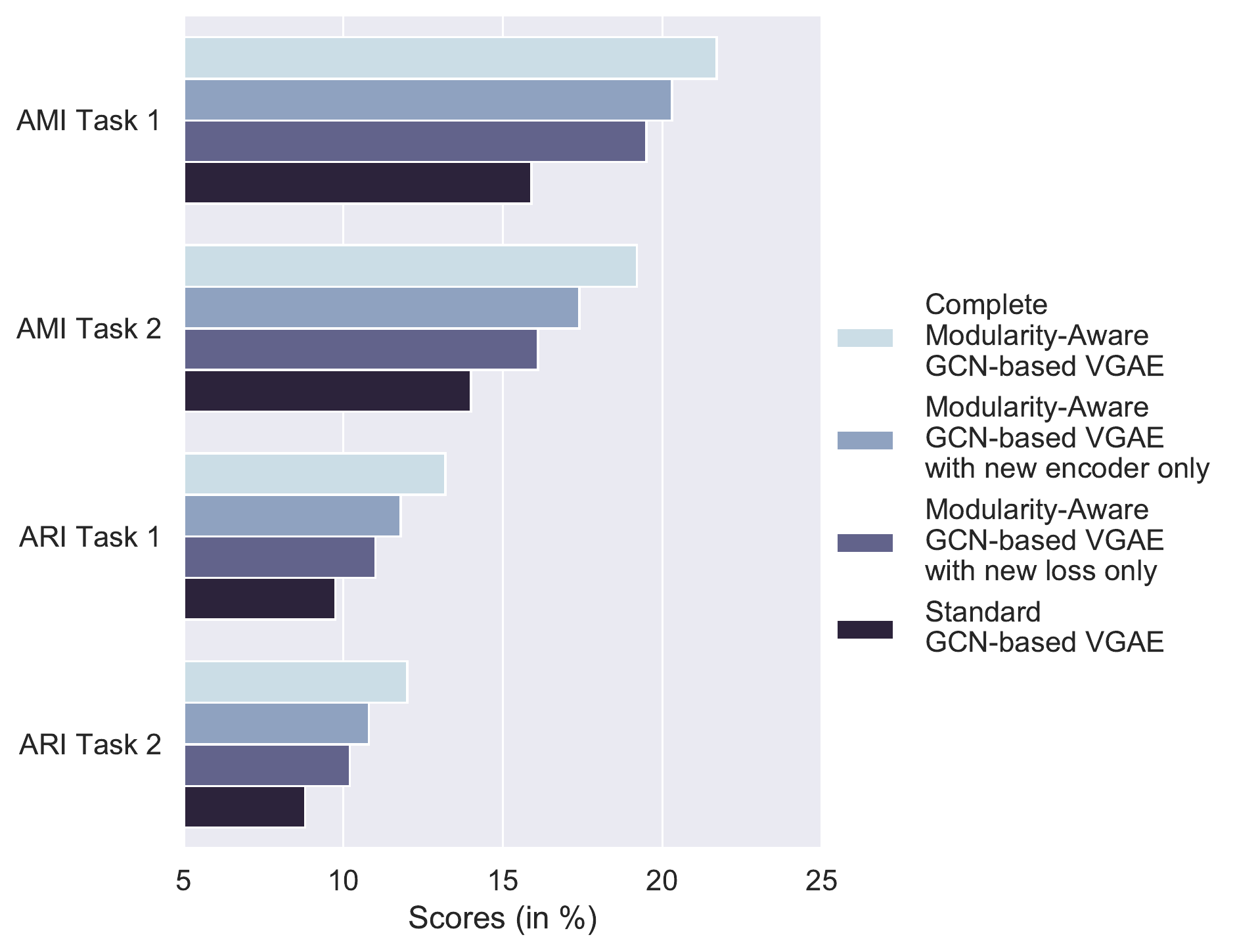}}}}
  \caption{Comparison of two ``complete'' Modularity-Aware VGAE, trained on (a) featureless Cora and (b) Deezer-Album with variants of these models only leveraging the new \textit{encoder} from Section~\ref{s32}, or the new \textit{loss} from Section~\ref{s33}. We observe that incorporating any of these two components improves community detection on these two graphs w.r.t. Standard VGAE. Moreover, using both components \textit{simultaneously} leads to the best results. Note, the optimal pair ($\lambda$, $\beta$) for complete models might differ from the optimal $\lambda$ (resp. $\beta$) when incorporating the new encoder (resp. loss) only.}
    \label{fig:ablation}
\end{figure*}

\subsubsection{Results for Joint Link Prediction and Community Detection (Task 2)}
\label{s423}

We now study results for Task 2, the joint link prediction and community detection task described in Section~\ref{s412}, and performed on incomplete versions of the graph datasets where 15\% of edges are randomly masked. Results for this task (i.e., AMI and ARI scores from community detection on incomplete graphs, and AUC and AP scores from link prediction on test sets) are reported in the four rightmost columns of Tables~\ref{tab:coraresults},~\ref{tab:pubmedresults}~and~\ref{tab:allresults}. AMI and ARI scores from this task are also included in the ablation study in Figure~\ref{fig:ablation}.

We draw several conclusions from these additional experiments. First of all, we confirm that AMI and ARI scores decrease slightly w.r.t. Task 1, which was expected due to the absence of part of the graph structure during the training phase (e.g., from 46.65\% to 42.86\% mean AMI, for Linear Modularity-Aware VGAE on Cora in Table~\ref{tab:coraresults}). Nonetheless, the ranking of the different methods under consideration remains consistent with Task~1. In particular, our Modularity-Aware models still outperform baselines in cases where they were already outperforming in Task~1 (e.g., with a top 43.48\% mean AMI for Linear Modularity-Aware GAE on Cora  in Table~\ref{tab:coraresults}, vs 28.41\% for the standard Linear GAE and 39.09\% for the Louvain method).

Besides these confirmations, the main goal of Task 2 was to address our second research question stated in Section \ref{introduction}: \textit{Do improvements on the community detection task necessarily incur a loss in the link prediction performance or can they be jointly addressed with high accuracy?} 
Indeed, as GAE and VGAE were originally recognized as effective link prediction methods (see Section~\ref{s2}), improving community detection while deteriorating link prediction might be undesirable, especially in problems requiring effective node embeddings for multitask applications (see the Deezer example from Section~\ref{s411}). By design, our proposed encoders, losses, and selection procedure specifically aimed to avoid such a deterioration.

Empirical results confirm the ability of Modularity-Aware models to preserve comparable link prediction performances w.r.t standard GAE and VGAE. For instance, in Table~\ref{tab:coraresults}, our Linear Modularity-Aware GAE reaches mean AUC and AP scores of 87.18\% and 88.53\%, respectively, which is comparable (or even slightly better in the case of AUC) to Linear Standard GAE (84.46\% and 88.42\%, respectively). We reach similar results for the three other Modularity-Aware models in Table~\ref{tab:coraresults}, while scores of several baselines deteriorate by a few points. Overall, all other Modularity-Aware models reported in the complete Table~\ref{tab:allresults} achieve comparable (either better, identical, or only a few points below) AUC and AP scores w.r.t. their GAE or VGAE counterparts.

\subsubsection{Limitations and Possible Extensions}
\label{s424}

As observed in Section \ref{s422}, empirical gains of Modularity-Aware models are less pronounced on graphs equipped with node features (although non-null in 2/3 cases). In Table~\ref{tab:allresults}, for Cora with features, we ``only'' report increases in Task 1 AMI scores of +2.63 points w.r.t. the corresponding standard VGAE model. For comparison, in the featureless case, we reported increases in Task 1 AMI scores of +11.53 points. Furthermore, our Modularity-Aware VGAE does not surpass the standard VGAE at all on Pubmed with features. We hypothesize that the incorporation of Louvain-based prior clusters in Modularity-Aware models might be less relevant for these attributed graphs. Indeed, while Louvain only leverages the graph structure for node clustering, node features seem to play a strong role in the identification of ground-truth communities for these graphs. 

Nevertheless, we recall that the use of the Louvain method was made without loss of generality. As explained in Section~\ref{s32}, our revised message passing operators would remain valid for other methods that alternatively derive a prior clustering signal. Future experiments on such alternatives (e.g., methods processing node features) could therefore improve community detection performances on these three attributed graphs. Overall, the empirical performance of our method directly depends on the quality of the underlying prior clustering method used to compute $\Ac$ and $\Ao$, which should therefore be carefully~selected. 

More broadly, our framework could also straightforwardly incorporate alternative encoders (besides linear and multi-layer GCN encoders), alternative decoders (e.g., decoders replacing inner~products by more refined graph reconstruction methods \cite{salha2019-2,aaai20}) and alternative losses (for instance, as explained in Section~\ref{s331}, our modularity-inspired regularizer could be optimized in conjunction with the ELBO loss from VGAECD/VGAECD-OPT \cite{choong2018learning,choong2020optimizing} involving Gaussian mixtures). One could also replace our $k$-means step, to cluster vectors $z_i,$ by another method such as $k$-medoids \cite{park2009simple} or spectral clustering \cite{von2007tutorial} (although our preliminary experiments in this direction did not reach significantly better results). Future work considering such alternative architectures for Modularity-Aware GAE and VGAE could definitely lead to the improvement of our models.

Lastly, we also plan to extend Modularity-Aware GAE and VGAE to dynamic graphs. Indeed, while our work considered fixed graph structures, real-world graphs often evolve over time. For instance, on the Deezer service, new albums should regularly appear in the musical catalog. New nodes will therefore appear in the Deezer-Album graph. Capturing such changes, e.g., through dynamic graph embedding methods \cite{hamilton2020graph}, might permit learning more refined representations and providing effective dynamic community detection.

\subsubsection{Towards More Comparisons to Non-GAE/VGAE Methods}
\label{s425}

Our experiments mainly aim to compare our proposed Modularity-Aware GAE and VGAE models to the Louvain method and to other autoencoders, i.e., 1) the standard GAE and VGAE models, and 2) alternative methods that improve community detection with GAE and VGAE. Giving us a total of 12 baseline models that we compare against. 
Nonetheless, while our scope is limited to autoencoders, we acknowledge that Modularity-Aware GAE and VGAE could also be seen as multi-task models, as we train them to perform community detection and link prediction simultaneously. As a consequence, in future research, it would be interesting to further study these models through the lens of the multi-task learning framework and thus compare them to other graph-based multi-task learning methods \cite{wu2022mtgcn,wang2020multi}.

In addition, there are also interesting parallels between our proposed method and pre-training methods for self-supervised learning (SSL) \cite{xie2021self,liu2021graph,wu2021self}, that might deserve further investigations in future research. There are, however, also several fundamental differences to be pointed out. More specifically, our proposed loss, as is common in the ``joint learning'' methods in the SSL literature, is also a combination of two loss terms stemming from different learning tasks. However, our method does not involve the distinction between a self-supervised or auxiliary and a downstream or main task; both tasks performed by our method are equally relevant and important. Specifically, with regards to the paradigm of pre-training an encoder on a self-supervised task and then fine-tuning it on a downstream task, our method is trained in a single-stage and does not involve pre-training of any kind. Also, in spirit, our effort in the Modularity-Aware GAE/VGAE is similar to the works of \citet{hu2019pre} incorporating clustering information in their cluster-preserving models, the self-supervised graph partition GCN by \citet{you2020does} and the models predicting degree, local node importance, and local clustering coefficient in \citet{jin2020self}. Also, the works addressing label scarcity problems by generating ``pseudo-labels'' are conceptually similar to our approach since pseudo-labeling often relies on clustering \cite{xie2021self,li2018deeper,sun2020multi}. However, unlike these works, we do not use the community labels in the loss function, but rather optimize the modularity of a clustering obtained from our embedding vectors. In addition, our models do not aim to predict the community membership of a node from its embedding vector, or any other node-level centrality measure such as the degree. Instead, they aim to estimate the likelihood of an edge between two nodes from their two embedding vectors. A more in-depth comparison of our method with this related literature has the potential to yield further improvements of the Modularity-Aware GAE and VGAE that we propose here.

\section{Conclusion}
\label{s5}

In this paper, we introduce a well-performing approach for simultaneous link prediction and community detection compatible with both the GAE and VGAE frameworks. 
This approach is based on a rigorous diagnosis of the shortcomings of existing approaches to this problem. 
Our approach takes advantage of two elements: A theoretically grounded variant of message passing operator in the GAE and VGAE encoders, that incorporates prior cluster information, and the addition of a modularity-based loss component to the usual existing loss functions. Both elements were experimentally shown to have an individual impact on the community detection performances. We furthermore introduce a revised hyperparameter selection procedure specifically designed for joint link prediction and community detection. We experimentally demonstrated the effectiveness of the approach on multiple datasets, including common benchmark graphs and large scale real-world ones, both with node features and, crucially featureless graphs: The results are consistently on par or better than the state-of-the-art for both link prediction and community detection.  This overcomes the common pitfall of most state-of-the-art methods that usually show high accuracy for either link prediction or community detection but not for both simultaneously.
Future research directions include the in-depth analysis of the utilization of the two loss terms and the relationship among them. Specifically, combining a dot-product similarity as a metric for the link prediction term and the Euclidean distance for the modularity-inspired term exhibited strong performance and a study on their behavior appears to be insightful.

\clearpage

\appendix
\section{Proofs of the Propositions in Section \ref{s324}}\label{app:spectral_results}
We begin by introducing several theoretical results which we will use in the majority of our proofs. The specific formulations of the results in this section, i.e., Definition \ref{def:matrix_poly} and Propositions \ref{thm:poly_transf}, \ref{thm:union_of_spectra} and \ref{thm:similar}, are adapted from \citet{Lutzeyer2020}.

When considering regular graphs, i.e., graphs containing only nodes of equal degree, their different graph representation matrices, such as the adjacency matrix, Laplacian matrices, and the GCN's message passing operator, are related via polynomial matrix transformations. These are now defined. 

\begin{definition} \label{def:matrix_poly}
\citet[p.~36]{Horn1985} define the evaluation of a polynomial $p(x) = c_l x^l + c_{l-1} x^{l-1} +\ldots + c_1 x + c_0$ at a matrix $\mati$ as
$$
p(\mati) =c_l \mati^l + c_{l-1} \mati^{l-1} +\ldots + c_1 \mati + c_0 I.
$$
\end{definition}

\citet{Horn1985} further discuss the influence of a polynomial matrix transformation on the matrices' eigenvalues and eigenvectors, which we reproduce below. 

\begin{proposition} \label{thm:poly_transf}
\citep{Horn1985} Let $p(\cdot)$ be a given polynomial. If $\evali$ is an eigenvalue of $\mati \in \mathbb{R}^{n \times n}$, while $\eveci$ is an associated eigenvector, then $p(\evali)$ is an eigenvalue of the matrix $p(\mati)$ and $\eveci$ is an eigenvector of $p(\mati)$ associated with $p(\evali)$.
\end{proposition}

Since we consider graphs consisting of several connected components in a multitude of our propositions, we now provide a theorem which relates the eigenvalues and eigenvectors of the whole graph to those of its connected components. 

\begin{proposition} \label{thm:union_of_spectra}
Let $\mathcal{G}$ be a graph with corresponding adjacency matrix $A$ and assume $\mathcal{G}$ to consist of $\ncluster$ connected components each with corresponding adjacency matrix $A_\clustersubscript$ for $\clustersubscript\in\{1, \ldots, \ncluster\}.$ Then, the eigenvalues of $\agcn{A}$ are equal to the union of the eigenvalues of $\agcn{A_\clustersubscript}$ over $\clustersubscript\in\{1, \ldots, \ncluster\}.$ Further, a set of eigenvectors of $\agcn{A}$ can be constructed from the eigenvector sets of $\agcn{A_\clustersubscript}$ for $\clustersubscript\in\{1, \ldots, \ncluster\}.$ 
\end{proposition}

\begin{proof}
For a graph $\mathcal{G}$ consisting of several connected components there exists a node ordering such that its corresponding adjacency matrix $A$ is block diagonal. Since, the addition of the identity matrix $I_n$ and the multiplication by a diagonal matrix $(D+I_n)^{-\frac{1}{2}}$ does not affect the block diagonal property of a matrix, the matrix $\agcn{A}$ is also block diagonal. 

Now, the characteristic equation of block diagonal matrices factorizes into polynomials corresponding to the different blocks \citep[p.~291]{Bernstein2009}. Therefore, the set of eigenvalues of any block diagonal matrix is equal to the union of the set of eigenvalues of matrices containing only the blocks. Consequently,  the eigenvalues of $\agcn{A}$ are equal to the union of the eigenvalues of $\agcn{A_\clustersubscript}$ over $\clustersubscript\in\{1, \ldots, \ncluster\}.$

Furthermore, any eigenvector of a given connected component, described by $\agcn{A_\clustersubscript},$ can be modified to be an eigenvector of $\agcn{A}$ by the insertion of zero values in all entries corresponding to nodes not contained in the connected component described by $\agcn{A_\clustersubscript}.$
\end{proof}

To allow us to relate the spectra and eigenvectors of the well studied adjacency matrix $A$ to the more novel GCN message passing operator $\agcn{A}$ we frequently make use of the matrix similarity relationship. The consequences of a matrix similarity relationship between matrices on their eigenvalues and eigenvectors is discussed in Proposition \ref{thm:similar}.

\begin{proposition}\label{thm:similar}
\citep[pp.~45,60]{Horn1985} If two matrices $\mati$ and $\matii$ are related via a nonsingular matrix $S$ as follows, $\mati = S^{-1}\matii S.$ Then, $\mati$ and $\matii$ have the same multiset of eigenvalues. Further, for eigenvector $v$ with corresponding eigenvalue $\evali$ of $\mati$ gives rise to an eigenvector $Sv$ of $\matii$ with equal corresponding eigenvalue $\evali.$
\end{proposition}

We can now begin to prove the propositions discussed in Section \ref{s324}.

\subsection{Proof of Proposition \ref{thm:Acspectrum}} \label{app:proof_Acspectrum}

\textbf{}\begin{proof}
\citet{Teke2017}, among many others, state that the unnormalised Laplacian matrix $L=D-A$ corresponding to a complete graph has eigenvalue $0$ with multiplicity $1$ and eigenvalue $n$ with multiplicity $n-1.$ 
Furthermore, the eigenspace corresponding to the eigenvalue $n$ is spanned by a 2-sparse set of orthogonal eigenvectors \citep{Teke2017}.
In addition, the eigenvector corresponding to the eigenvalue $0$ of the unnormalised graph Laplacian describing a connected graph is well known to be the constant eigenvector \citep{von2007tutorial}. Since the complete graph is regular, its degree matrix is a multiple of the identity matrix, i.e., $D=(n-1)I_n.$ Therefore, for complete graphs the following relationship holds $\agcn{A} = I_n - \frac{1}{n}L.$ Hence, from Proposition \ref{thm:poly_transf} the eigenvectors of $\agcn{A}$ and $L$ are equal and $\agcn{A}$ has the eigenvalue $1$ with multiplicity 1 and eigenvalue $0$ with multiplicity $n-1.$

Now, since $\Ac$ corresponds to a graph composed of several complete graphs, we can invoke Proposition \ref{thm:union_of_spectra} to construct the spectrum and eigenvectors of $\agcn{\Ac}$ from the the spectrum and eigenvectors of $\agcn{A}$ corresponding to a complete graph, which we just derived. Consequently, $\agcn{\Ac}$ has eigenvalues $\{\{1\}^{\ncluster}, \{0\}^{n-\ncluster}\}$ and a set of eigenvectors as described in the statement of Proposition \ref{thm:Acspectrum}.




\end{proof}

\subsection{Proof of Proposition  \ref{thm:spectral_relation_A_Ac_AAc}} \label{app:proof_spectral_relation_A_Ac_AAc}

\begin{proof}
Proposition \ref{thm:union_of_spectra} can be used to extend the required result from one connected component of $\mathcal{G}$ to the full graph. Therefore, we consider only one connected component on the graph from now on. 
Let $\Acomp$ denote the adjacency matrix of this connected component, containing nodes of degree $b,$ and $\Accomp$ denote the corresponding complete component of $\Ac,$ containing nodes of degree $\ncomp-1.$ Then, the adjacency matrix of our connected component under consideration $\Acomp+\weight \Accomp$ is related to $\agcn{\Acomp + \weight \Accomp}$ as follows, $$\agcn{\Acomp+\weight \Accomp} = \frac{1}{b+\lambda (\ncomp-1) +1} \left(\Acomp+\weight \Accomp+I\right).$$ 
Therefore, from Proposition \ref{thm:poly_transf} it follows that $\Acomp+\weight \Accomp$ and $\agcn{\Acomp + \weight \Accomp}$ share eigenvectors. Similarly, the relations $\agcn{\Acomp} = \frac{1}{b+1} (\Acomp+I)$ and $\agcn{\Accomp} = \frac{1}{\ncomp} (\Accomp+I)$ together with Proposition \ref{thm:poly_transf} allow us to establish that both $\agcn{\Acomp}$ and $\Acomp$ as well as $\agcn{\Accomp}$ and $\Accomp$ each have a common set of eigenvectors.

We now make use of a result by \citet[p.~25]{Godsil1993}, which states that the adjacency matrix of a graph commutes with the matrix of all ones, i.e., $\Ac+I_n,$ if and only if the graph under consideration is regular. Further, a family of matrices is a commuting family if and only if they are simultaneously diagonalizable, i.e., they share a set of eigenvectors \citep[p.~52]{Horn1985}.
Hence, $\Acomp$ and $\Accomp$ share a set of eigenvectors. Furthermore, this shared set of eigenvectors is also a valid set of eigenvectors for $\Acomp+\weight \Accomp,$ which, in conjunction with the above polynomial relationships, establishes the needed eigenvector relation.

In addition the eigenvalues of the sum of two simultaneously diagonalizable matrices $\mati, \matii$ with eigenvalues denoted by $\evali$ and $\evalii,$ respectively, are related \citep[p.~54]{Horn1985}  as follows
\begin{equation}\label{eq:diagonalisable_eval_relation}
\mathcal{S}(\mati +\matii) = \{\evali_1 + \evalii_{s(1)}, \ldots, \evali_n + \evalii_{s(n)}\},
\end{equation}
for some permutation $s(\cdot)$ defined on the set $\{1, \ldots, n\}.$ Since $\Acomp$ and $\Accomp+I_n$ are simultaneously diagonalizable their eigenvalues follow the relation in Equation \eqnref{eq:diagonalisable_eval_relation}. Now the above polynomial relationships of the GCN message passing operators to the corresponding adjacency matrices gives us the desired eigenvalue result and establish that $g_1(\mu) = \frac{b+1}{b+\lambda (\ncomp-1) +1}(\mu-1)$ and $g_2(\eta) = \frac{\lambda(\ncomp-1)+1}{b+\lambda (\ncomp-1) +1}(\eta-1)+1.$  










\end{proof}

\subsection{Proof of Proposition \ref{thm:Aospectrum}}\label{app:proof_Aospectrum}
\begin{proof}

\citet[p. 453]{Hoory2006} state that for $\dregc$-regular graphs the largest eigenvalue of the corresponding adjacency matrix $\Acomp$ equals $\dregc$ and the corresponding eigenvector is constant. Now the relation $\agcn{\Acomp} = \frac{1}{o+1}\left(\Acomp +I_n\right)$ in conjunction with Proposition \ref{thm:poly_transf} establish that the largest eigenvalue of $\agcn{\Acomp}$ equals 1 with a corresponding constant eigenvector. This spectrum and eigenvectors can be extended to the matrix $\agcn{\Ao}$ corresponding to a graph of several $\dregc$-regular connected components using Proposition \ref{thm:union_of_spectra}. The comparison of the derived spectrum and eigenvectors to those derived in Proposition \ref{thm:Acspectrum} completes this proof.


\end{proof}

\clearpage

\bibliography{references}

\end{document}